\documentclass[letterpaper]{article} 
\usepackage{aaai2027}
\nocopyright
\usepackage[hyphens]{url}  
\usepackage{graphicx} 
\usepackage{natbib}  
\usepackage{caption} 
\usepackage{booktabs}
\usepackage{amsmath}
\usepackage{amssymb}
\usepackage{amsthm}
\newcommand{\E}{\mathbb{E}}
\newcommand{\R}{\mathbb{R}}
\newcommand{\KL}{\mathrm{KL}}
\newcommand{\sgm}{\sigma}
\newcommand{\ind}[1]{\mathbf{1}\{#1\}}
\newcommand{\Bern}{\mathrm{Bern}}
\DeclareMathOperator{\sign}{sign}

\newtheorem{theorem}{Theorem}
\newtheorem{proposition}{Proposition}
\newtheorem{lemma}{Lemma}
\newtheorem{corollary}{Corollary}
\theoremstyle{remark}
\newtheorem{remark}{Remark}

\theoremstyle{plain}
\newtheorem{sproposition}{Proposition}[section]
\newtheorem{stheorem}[sproposition]{Theorem}

\theoremstyle{remark}
\newtheorem{sremark}[sproposition]{Remark}

\newcommand{\corresponding}{\textsuperscript{\rm *}}

\title{Abstention as an Action Can Kill Both the Reward Gradient and the KL Anchor:\\ Collapse Law and Repair for Error-Penalized Reinforcement Learning}
\author{
    Xujun Che\textsuperscript{\rm 1\dag}\corresponding,
    Yuchen Yuan\textsuperscript{\rm 2\dag},
    Weida Zhao\textsuperscript{\rm 3},
    Chenyang Yu\textsuperscript{\rm 4}
}
\affiliations{
    \textsuperscript{\rm 1}Department of Cybersecurity, University of North Carolina at Charlotte, Charlotte, NC, USA\\
    \textsuperscript{\rm 2}Department of Information Sciences and Technology, George Mason University, Fairfax, VA, USA\\
    \textsuperscript{\rm 3}Department of Computer Science and Software Engineering, Auburn University, Auburn, AL, USA\\
    \textsuperscript{\rm 4}Department of Computer Science and Engineering, University of North Texas, Denton, TX, USA\\
    xche@charlotte.edu, yyuan21@gmu.edu, wzz0050@auburn.edu, chenyangyu@my.unt.edu
}

\begin{document}

\maketitle
\begingroup\renewcommand\thefootnote{\dag}\footnotetext{These authors contributed equally.}\endgroup

\begin{abstract}
Error-penalized scoring rules ($+1$ for a correct answer, $-\lambda$ for a wrong one, $0$ for abstaining) are increasingly prescribed against hallucination: a rational agent facing such a rule answers exactly when its correctness probability exceeds Chow's threshold $t^\ast=\lambda/(1+\lambda)$. We prove that a KL-anchored gradient learner can do the opposite. When abstention is a discrete action, the reward gradient and the anchor's restoring force are throttled by the same gate-saturation factor and die together: under explicit conditions (among them, blanket answering loses score in expectation and prompts share a bounded readout) the model drifts toward refusing everything, its mean training reward rising to zero like $1/t$ in training time $t$, so the curve reads as improvement while coverage collapses. The advantage estimator compounds the failure: in its sparse-answer regime, group normalization silently replaces every designed penalty with an effective penalty of one, moving the learned threshold from $t^\ast$ to $1/2$. The repair is structural: train a mandatory confidence report with a strictly proper score plus a correctness reward, and abstain only at deployment by thresholding the report. The always-emitted report has no gate to saturate, so no shared factor can kill its reward gradient and its anchor together, and its calibrated optimum is attracting. Simulations confirm every prediction, and experiments on language models at two scales confirm the mechanism live: the rule silences questions the models demonstrably still solve within ten optimizer steps, an ablation isolates the cause, and report-level training raises coverage, accuracy, and calibration together.
\end{abstract}

\section{Introduction}\label{sec:intro}

Language models hallucinate in part because we pay them to. \citet{kalai2025why} argue that binary-graded training and evaluation reward confident guessing over honest abstention, and prescribe the classical remedy: penalize errors, give partial or zero credit to abstention, and state the stakes. The prescription is spreading through evaluation practice, and a growing line of work adopts it directly as a reinforcement-learning objective, rewarding answers $+1/{-\lambda}/0$ \citep{truthrl2025,song2025hallucination,xu2025reliability}.

For a \emph{rational} agent the prescription is provably correct: facing the rule $(+1,-\lambda,0)$, the optimal policy answers exactly when its correctness probability exceeds Chow's threshold $t^\ast=\lambda/(1+\lambda)$ \citep{chow1970}. This paper concerns what the same rule does to a \emph{gradient learner}: a policy dragged toward the payoff table by policy-gradient fine-tuning with a KL anchor to its base model, the standard RLHF configuration \citep{christiano2017,stiennon2020,ouyang2022}. We show the two can produce opposite outcomes for policies whose prompts share a bounded readout, and that the difference has a clean mechanical cause.

\paragraph{The mechanism in one identity.}
\looseness=-1 Under a penalty rule, abstention is a discrete \emph{action}: the model answers with probability $\sgm(v(x))$, where $v$ is a learned gate logit and $\sgm$ the logistic function. Every gradient of the expected score with respect to the gate parameters carries the saturation factor $\sgm'(v)$. So far this is ordinary softmax-policy-gradient saturation \citep{mei2020,razin2024}. The new observation concerns the anchor. For Bernoulli gates,
\begin{equation}\label{eq:klgrad}
\frac{d}{dv}\,\KL\big(\Bern(\sgm(v))\,\big\|\,\Bern(\sgm(v_0))\big)\;=\;\sgm'(v)\,(v-v_0),
\end{equation}
so the KL anchor's restoring force carries the \emph{same} factor. If the base model answers nearly everything and blanket answering loses score in expectation (precisely the situation the penalty rule is designed to punish), then $v$ falls, $\sgm'(v)$ dies, and the anchor's pull back toward the base dies with it. The model does not converge to calibrated threshold abstention; it drifts toward answering nothing, ever more slowly, along a $\beta$-independent power law, onto a plateau that is metastable rather than absorbing: escape exists, but arrives exponentially late in the gain it must accumulate. Practitioners' intuition that ``the KL term keeps the policy near its base'' fails structurally: KL regularization, ordinarily a stabilizer \citep{geist2019,vieillard2020}, provides no restoring force on exactly the coordinate that is collapsing (the reverse KL would, but it is not what RLHF uses).

\paragraph{Contributions.}
\begin{enumerate}\itemsep1pt
\item \textbf{Mechanism.} A single factor, the gate's saturation $\sgm'(v)$, throttles the reward gradient and the KL anchor's restoring force alike: the two die together at the all-abstain vertex. Collapse needs three ingredients (blanket answering loses score, prompts share a bounded readout, abstention is action-level); removing any one eliminates the collapse in the analyzed class. The mechanism lifts from the two-token gate to sequence policies with branch-separable content.
\item \textbf{Law.} Under these conditions the collapse follows a finite-time law: the mean training reward rises to zero like $1/t$ up to log factors, at a time stable in the anchor strength.
\item \textbf{Estimator.} The advantage estimator selects the law: group normalization steepens the decay to $1/t^2$ above a knee at answer rate ${\approx}1/G$ ($G$ the group size) and, below it, silently replaces every designed penalty $\lambda$ by $\lambda_{\mathrm{eff}}=1$; dynamic resampling conserves the collapse per rollout, yielding an estimator-indexed menu of falsifiable predictions.
\item \textbf{Repair.} Move abstention from the action space to a mandatory confidence report, trained with a strictly proper score plus a correctness reward and thresholded only at deployment: the report channel has no shared vanishing factor, and its calibrated optimum is interior and attracting.
\item \textbf{Evidence.} Enumeration and simulation confirm every prediction; live tests on a 1.5B model reproduce the collapse on questions the model provably still solves, isolate its cause by ablation, and measure both gradient norms at zero there, while a single-seed 7B run adds the predicted late escape.
\end{enumerate}

\section{Setting}\label{sec:setting}

Prompts $x$ carry a correctness probability $q(x)\in[0,1]$ for the frozen base model's best candidate answer. A decision layer observes an internal signal $s$ with marginal density $p(s)$ and posterior $\bar q(s)=\E[q\mid s]$. Under the penalty rule $(+1,-\lambda,0)$ the expected score of answering at $s$ is $m(\bar q(s))$, where $m(u):=(1+\lambda)u-\lambda$, and $g(s):=p(s)\,m(\bar q(s))$ is the \emph{answering-gain density}. The optimal deterministic policy answers iff $g(s)>0$, i.e.\ iff $\bar q(s)>t^\ast=\lambda/(1+\lambda)$ (Chow's rule), attaining $U^\ast=\int g_+$ with rational coverage $\mathrm{cov}^\ast=\int_{g>0}p$.

\paragraph{Two mechanisms, matched heads.}
Both mechanisms read the same bounded feature $\phi(s)=\sgm(ws+a)\in(0,1)$ through a four-parameter head; boundedness reflects how common confidence readouts (sigmoid or softmax heads, verbalized scales) are built.

\emph{Action-level (penalty rule).} Gate logit $v_\theta(s)=c_4\phi(s)+c_0$ with $\theta=(w,a,c_4,c_0)$; during training the model \emph{samples} answer/abstain with probability $\sgm(v_\theta(s))$, and the objective is
\begin{align}
J_\beta(\theta)\;=\;&\int \sgm(v_\theta)\,g\,ds\notag\\
&-\beta\int p\;\KL\big(\Bern(\sgm(v_\theta))\,\|\,\Bern(\sgm(v_{0}))\big)\,ds,
\label{eq:objective}
\end{align}
with $v_0:=v_{\theta_0}$ the base gate and $\beta>0$ the anchor strength. The KL in \eqref{eq:objective} is the gate term of the sequence-level KL (the chain rule adds branch-conditional terms that vanish when branch content matches the base); Theorem~\ref{thm:lift} treats the full sequence KL with trainable content.

\emph{Report-level (proper rule).} Confidence $c_{\theta'}(s)=\kappa_0+(\kappa_1-\kappa_0)\phi(s)$ with $\theta'=(w,a,\kappa_0,\kappa_1)$ and $\kappa_0,\kappa_1\in[0,1]$, so $c\in[0,1]$ (interior optima in our calibrations; a projected flow preserves the box at the same rates toward the constrained optimum); the model answers \emph{every} question, reports $c$, and is scored by the Brier rule $1-(Y-c)^2$, $Y\sim\Bern(\bar q)$. Its per-question output is deterministic given $s$, so no action-policy KL exists; we anchor it with the proximal term $\tfrac\beta2\|\theta'-\theta'_0\|^2$ (function-space anchors on the report leave every conclusion unchanged; Proposition~\ref{prop:report}).

\emph{Deployment.} Both greedy: the penalty policy answers iff $v_\theta(s)>0$ and the report policy iff $c_{\theta'}(s)>t^\ast$; both target the same decision region.

\emph{Initialization.} Matched, overconfident, weakly discriminative: the base answers ${\approx}88\%$ of questions while the true posterior spans $[0.2,0.9]$ \citep{guo2017,tian2023}.

\paragraph{Standing assumption and scope.}
\textbf{(B1)}: $D_0:=-\int g\,ds>0$; equivalently $\E[q]<t^\ast$: blanket answering loses score in expectation. It is necessary for collapse. Two further conditions are necessary, and we state them as sharply as the first.

\begin{proposition}[Tabular policies do not collapse]\label{prop:tabular}
Under vanilla policy gradient with one free logit per prompt, the flow $\dot v(s)=\sgm'(v(s))[\,m(\bar q(s))-\beta(v(s)-v_0(s))\,]$ has stationary point $v^\ast(s)=v_0(s)+m(\bar q(s))/\beta$. On every prompt with $m(\bar q(s))\neq0$ the equilibrium decision converges to Chow's rule as $\beta\downarrow0$; at finite $\beta$ it is correct outside a band of width $O(\beta)$ in $m$ around the threshold (uniformly, whenever $|m|\ge\gamma>\beta\|v_0\|_\infty$).
\end{proposition}

The population condition (B1) can drag profitable prompts down only because the bias $c_0$ is common to all of them. The paper's thesis is the conjunction: \emph{a discrete abstain action supplies a vertex that attracts on the $\bar q<t^\ast$ prompts; a shared bounded readout transmits the attraction to the $\bar q>t^\ast$ prompts; and the KL anchor, which is supposed to arrest exactly such drifts, is throttled by the same factor that drives them.}

\section{Anatomy of the Collapse}\label{sec:collapse}

We take the collapse apart in three passes: the static mechanism at the abstain vertex, the finite-time law it forces under gradient flow, and the estimator dependence that sets the law's exponent and its effective penalty.

\subsection{One Factor Throttles Both Forces}\label{sec:mech}

\begin{lemma}[Shared throttling factor]\label{lem:klgrad}
For the objective \eqref{eq:objective},
\[
\nabla_\theta J_\beta=\!\int\! \sgm'(v_\theta)\Big[g-\beta p\,(v_\theta-v_{0})\Big]\nabla_\theta v_\theta\,ds .
\]
Reward and anchor share the factor $\sgm'(v_\theta)$: where the gate saturates, both vanish (differentiate \eqref{eq:objective} under the integral; the KL term is \eqref{eq:klgrad} pointwise).
\end{lemma}

\begin{figure}[t]\centering
\includegraphics[width=0.97\linewidth]{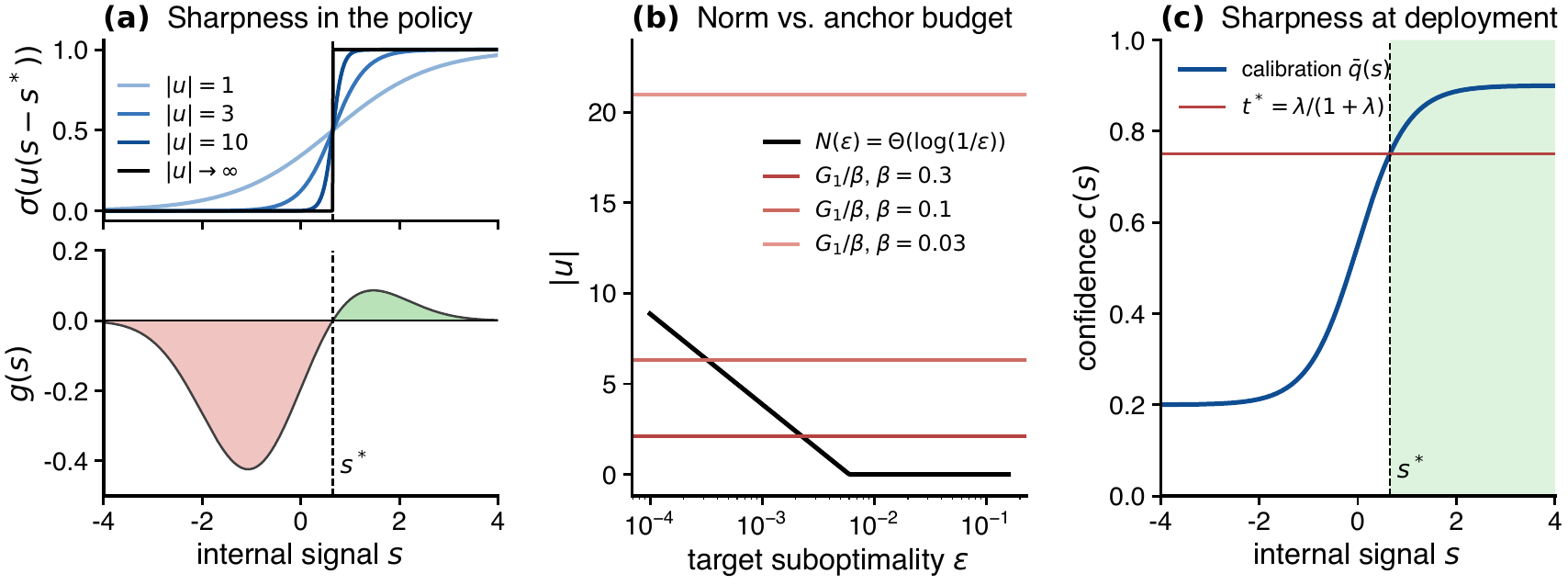}
\caption{The geometry of the failure. (a)~A hard threshold needs diverging logit norm; the loss lobe of $g$ outweighs its gain lobe (B1). (b)~Sharpness is free under the true KL, not under a proximal anchor. (c)~The report level needs an $O(1)$-norm fit, its sharpness supplied by $t^\ast$.}
\label{fig:mechanism}
\end{figure}

\begin{proposition}[The anchor cannot exclude the vertex]\label{prop:static}
Bernoulli KL to a fixed base is bounded, and at the all-abstain vertex ($\sgm(v)\equiv0$) it costs exactly $\overline{\KL}_\perp:=\E_p[\log\frac{1}{1-\sgm(v_0)}]<\infty$. Hence, under \textup{(B1)}, $J_\beta(\text{all-abstain})-J_\beta(\text{base})=L(0)-\beta\,\overline{\KL}_\perp>0$ whenever $\beta<\beta_{\mathrm{crit}}:=L(0)/\overline{\KL}_\perp$, where $L(0)=\int\sgm(v_0)(-g)$ is the base policy's expected loss. Total abstention strictly beats the base for every anchor strength below an explicit threshold, with no dynamics required. Because $\overline{\KL}_\perp$ grows only logarithmically in the base answer rate, $\beta_{\mathrm{crit}}$ sits above typical RLHF coefficients ($\beta\in[10^{-3},5\times10^{-2}]$, with reward in the rule's own units) across realistic calibrations.
\end{proposition}

\begin{proposition}[The dichotomy at the vertex is local]\label{prop:vertex}
Fix a prompt with posterior $\bar q$; evaluate the reward gradient at the degenerate vertex of the output's range.
\textup{(a)} \emph{Action level:} $\partial_v\E[r]=\sgm'(v)\,m(\bar q)\to0$ as $v\to-\infty$, with the sign of $m$. The abstain vertex is \emph{attracting on exactly the prompts with $\bar q<t^\ast$} and repelling on the rest.
\textup{(b)} \emph{Report level:} for any strictly proper score with Savage representation $\E[S(c)]=\Gamma(c)+\Gamma'(c)(\bar q-c)$, $\partial_c\E[S]=\Gamma''(c)(\bar q-c)$. The expected score strictly increases from the vertex toward $\bar q$, so the confidence vertex is \emph{never} attracting, and it is strictly repelling with gradient $\ge \zeta\bar q$ whenever $\Gamma''\ge \zeta>0$ near the boundary (Brier: $\Gamma''\equiv2$).
\end{proposition}

\begin{corollary}[The regularizer class that dies at the boundary]\label{cor:policyreg}
Let $R(\pi_\theta)$ be any functional of the policy with $\pi\,\partial_\pi R\to0$ as $\pi\to0$. Then $\partial_v R=\pi(1-\pi)\,\partial_\pi R\to0$: its gradient vanishes on the absorbing boundary with the reward's. The class contains the entropy bonus and RLHF's anchor $\KL(\pi_\theta\|\pi_0)$ (both logarithmic), but \emph{not} the reverse KL $\KL(\pi_0\|\pi_\theta)$, whose vertex cost diverges and which floors the answer rate at $\Theta(\beta)$: the direction of the KL is load-bearing, and RLHF's points the collapsing way. Balancing throttled drift against throttled restoring force leaves only an exponentially small floor, $P_{\mathrm{floor}}\asymp e^{-D_{\mathrm{eff}}/\beta}$ (e.g.\ ${\approx}10^{-13}$ at $\beta=0.01$), reached only after time $e^{\Theta(1/\beta)}$.
\end{corollary}

Together the statics make the vertex cheap, locally attracting, and unrestorable (Figure~\ref{fig:mechanism} draws the geometry); what they do not give is the \emph{rate}: what an experimenter would see, and on what clock.

\subsection{The Collapse Law under Gradient Flow}\label{sec:dynamics}

Write $J(\theta):=\int\sgm(v_\theta)\,g\,ds$ for the reward term of \eqref{eq:objective}, and $L(t):=-J(\theta_t)$. Since an answered prompt earns $m(\bar q)$ and an abstained one earns $0$, $L$ is \emph{minus the mean training reward}, the curve every RL run already logs. Define the observable \emph{anchor share} $\rho(t):=\beta\,\nabla J\!\cdot\!\nabla\mathcal K/\|\nabla J\|^2$, with $\mathcal K$ the population KL term of \eqref{eq:objective}, and the weighted gate mean $\bar\sgm(t):=\int\sgm^2(v)(-g)\big/\!\int\sgm(v)(-g)$. Let $r(s):=c_4\phi(s)$, $\Lambda(t):=\E_p[e^{r}]$ (the \emph{tilt}), $D_{\mathrm{eff}}(t):=\int e^{r}(-g)$, and let $t_{\mathrm{tilt}}$ be the zero of the tilted margin $\bar m_{\mathrm{tilt}}:=\E_{\mathrm{tilt}}[m(\bar q)]$ under the measure $\propto e^{r}p$.

\begin{theorem}[Collapse law on the invariant]\label{thm:onet}
Consider gradient flow $\dot\theta=\nabla J_\beta$ on \eqref{eq:objective} under \textup{(B1)}, from an initialization with $L(0)>0$.
\textup{(i)} \emph{(Envelope; one observable hypothesis.)} If $\rho(\tau)\le\tfrac12$ on $[0,t]$, then
\[
\frac{1}{L(t)}\;\ge\;\frac{1}{L(0)}+\frac12\int_0^t\big(1-\bar\sgm(\tau)\big)^2d\tau .
\]
\textup{(ii)} \emph{(Two-sided finite-time law.)} Assume additionally, on a window $[T_0,\,c_\star t_{\mathrm{tilt}}]$ with $c_\star<1$: the parameter-space certificate $\beta\|\theta_\tau-\theta_0\|_\infty\le\tfrac12 D_{\mathrm{eff}}(\tau)$; the measured sign $c_4(\tau)\le0$; a tilt bound $\Lambda(\tau)\le\bar\Lambda<1$; and the explicit small-$\beta$ condition
$\beta\,|c_4(0)|\,\bar\Lambda\le\tfrac14(1-\bar\Lambda)\,\underline D(c_\star)$,
where $\underline D(c_\star):=\min_{[0,c_\star t_{\mathrm{tilt}}]}D_{\mathrm{eff}}$. Then $\bar\sgm\to0$ in Ces\`aro mean, the readout gain grows as $|c_4|=O(\log t)$, and
\[
\frac{c_1}{t\log^2 t}\;\le\;L(t)\;\le\;\frac{c_2}{t},
\qquad T_0\le t\le c_\star t_{\mathrm{tilt}},
\]
with $c_1,c_2$ depending only on tracked quantities of the run. \textup{(iii)} \emph{(Endpoint.)} $L$ crosses zero at a time $Z_L$ within $O(1)$ absolute time of $t_{\mathrm{tilt}}$, provided an explicit error-dominance condition holds at the crossing (it does on our runs); $Z_L$ stabilizes as $\beta\downarrow0$ (measured: $3.2\%$ over two decades).
\end{theorem}

\begin{proof}[Proof idea]
The exact softmax identity gives $\partial_{c_0}J=-(1-\bar\sgm)L$; under $\rho\le\tfrac12$, $\dot L\le-\tfrac12\|\nabla J\|^2\le-\tfrac12(1-\bar\sgm)^2L^2$, which integrates to (i). For (ii), the plateau approximation $\sgm'(v)\asymp e^{v}$ turns the bias flow into $\frac{d}{dt}e^{-c_0}\ge D_{\mathrm{eff}}/2$, and a gain-growth lemma derives $|c_4|=O(\log t)$ from the sign, tilt, and small-$\beta$ inputs rather than assuming it. Full proofs of all results are in the appendix.
\end{proof}

\looseness=-1 Two readings matter. \emph{First}, the hypotheses are measurements, not conveniences: $\rho$ is computable from logged gradients (on our runs it peaks at $0.008$), and the sign, tilt, and small-$\beta$ inputs are checkable per run; every hypothesis held on every run reported in this paper, simulated and live. Read this way, part (i) is an a priori envelope from a single observable, and part (ii) is a certified finite-time law for any run whose logs pass its checks. \emph{Second}, the law is on $L$; the answer rate $P$ obeys $P=L/|\bar m_{\mathrm{tilt}}|\cdot(1+o(1))$ away from the endpoint, and its exponent is \emph{not} constant, so predictions should be stated on $L$.

\begin{proposition}[The proximal surrogate understates the pathology]\label{prop:prox}
Replace the KL in \eqref{eq:objective} by $\tfrac\beta2\|\theta-\theta_0\|^2$. The restoring force no longer carries $\sgm'(v)$; the flow converges exponentially to a stable equilibrium with answer-rate floor $P_\infty=\Theta(\beta\log(1/\beta))$. Analyses that model RLHF's KL by a proximal term are therefore \emph{not conservative} here: true KL is strictly more permissive of collapse than its weight-space stand-in.
\end{proposition}

\begin{proposition}[Bounded readouts delay greedy recovery]\label{prop:bounded}
With $v(s)=c_4\phi(s)+c_0$, $\phi\in[0,1]$, greedy decoding answers a positive-measure set iff the gauge-invariant scalar $\sup_s v(s)=\max(c_4,0)+c_0$ is positive. The tilt available to the flow is $\Lambda(t)\le e^{|c_4|}$; under the gain growth of Theorem~\ref{thm:onet}\textup{(ii)}, reaching gain $M$ takes time $e^{\Omega(M)}$. A bounded readout thus delays escape exponentially in the gain the escape must accumulate, without creating an absorbing state. Gates driven by unbounded hidden-state projections sit outside the class and escape.
\end{proposition}

\paragraph{Lifting beyond the frozen gate.}
The lifting theorem removes both restrictions at once: full sequence KL and trainable content, for mixture-structured policies.

\begin{theorem}[Gate-to-sequence lifting; abridged, with full statement and proof in the appendix]\label{thm:lift}
Let $\pi_\theta(y\mid x)$ be a sequence policy with refusal event $A$, answer and refusal probabilities $\pi_{\mathrm{ans}}=1-\pi_\theta(A\mid x)$ and $\pi_{\mathrm{abs}}=\pi_\theta(A\mid x)$, gate logit $v$, answer-branch correctness $q_\theta$, and answer-/refusal-conditional KLs $K_{\mathrm{ans}},K_{\mathrm{abs}}$ to base.
\textup{(i)} Exactly, $\KL_{\mathrm{seq}}=\KL_{\mathrm{gate}}+\pi_{\mathrm{ans}}K_{\mathrm{ans}}+\pi_{\mathrm{abs}}K_{\mathrm{abs}}$, and for any mixture parameterization
\[
\partial_v J_\beta=\sgm'(v)\big[m(q_\theta)-\beta(v-v_0)-\beta(K_{\mathrm{ans}}-K_{\mathrm{abs}})\big]:
\]
\emph{every} gate-channel term is throttled by $\sgm'(v)$, and the anchor taxes only the branch taken: abstaining \emph{hides} content drift from the KL.
\textup{(ii)} With branch-separable content (answer and refusal content on separate parameters, refusals initialized at base, and an explicit KL-gradient regularity bound $C_K<\infty$), the refusal branch is invariant and the collapsed manifold is stationary in the limit $\pi_{\mathrm{ans}}\to0$; the surviving restoring force pins \emph{refusals} to base, never answering.
\textup{(iii)} Under a capability ceiling preserving \textup{(B1)}, or under an explicit race inequality between the collapse clock and content drift, checkable on logged runs and stated in the appendix, the finite-time law of Theorem~\ref{thm:onet} survives with lower bound weakened by one logarithm, $\Omega(1/(t\log^3 t))$. The correctness channel is throttled by $\pi_{\mathrm{ans}}$, so the plateau delivers only $O(\log t_{\mathrm{tilt}})$ of total content signal: \emph{the plateau starves the very channel that could avert it}. Content shared across branches, and gate--content parameterizations with no bias coordinate, remain open; a generic single-decoder Transformer sits outside on both counts.
\end{theorem}

\subsection{The Advantage Estimator Rewrites the Rule}\label{sec:estimator}

Theorem~\ref{thm:onet} is a statement about vanilla policy gradient (PG). Practice uses group-based estimators: sample $G$ completions per prompt, form advantages $\hat A_i=(r_i-\hat\mu_r)/\hat\sigma_r$ (GRPO; \citealp{shao2024deepseekmath}) or mean-baseline advantages (RLOO, Dr.~GRPO; \citealp{ahmadian2024,liu2025drgrpo}). Because $\hat\sigma_r$ itself depends on the answer probability $p=\sgm(v)$, normalization changes the \emph{exponent} and, in the sparse regime, the \emph{rule}. Write $M_2:=\bar q+(1-\bar q)\lambda^2$.

\begin{proposition}[Drift menu]\label{prop:menu}
The per-prompt expected drift of the gate coordinate satisfies:
\textup{(a)} vanilla PG and mean-baseline estimators: $\Theta(p\,m)$, hence local $L$-exponent $-1$;
\textup{(b)} group-std normalization, $1/G\ll p\ll1$: $\approx m\sqrt{p/M_2}=\Theta(\sqrt p)$, hence exponent $-2$;
\textup{(c)} group-std normalization, $p\ll1/G$: a group contains at most one answer, whose advantage is $\sign(r)\sqrt{G-1}$ \emph{independently of $|r|$}, so the drift is $p\,(2\bar q-1)\sqrt{G-1}$: the designed penalty is erased, $\lambda_{\mathrm{eff}}=1$ for every nominal $\lambda$, and the exponent returns to $-1$. The knee between \textup{(b)} and \textup{(c)} sits at $p\approx1/G$.
\end{proposition}

\begin{table}[t]\centering\small
\setlength{\tabcolsep}{2.9pt}
\begin{tabular}{@{}llccc@{}}
\toprule
estimator & regime & drift $\propto$ & $L(t)$ & $L(N)$\\
\midrule
PG / RLOO & all $p$ & $p\,m$ & $\tilde\Theta(1/t)$ & $\tilde\Theta(G/N)$\\
GRPO & $\tfrac1G\ll p\ll1$ & $\sqrt{p}\,\tfrac{m}{\sqrt{M_2}}$ & $\Theta(1/t^2)$ & $\Theta(\tfrac{G^2}{N^2})$\\
GRPO & $p\ll\tfrac1G$ & $p(2\bar q{-}1)\sqrt{G{-}1}$ & $\tilde\Theta(1/t)$ & $\tilde\Theta(\tfrac{G}{kN})$\\
+DAPO & $p\ll\tfrac1G$ & $\tfrac{(2\bar q-1)\sqrt{G-1}}{G}$ & $\Theta(e^{-\omega t})$ & $\tilde\Theta(\tfrac{G}{kN})$\\
\bottomrule
\end{tabular}
\caption{The estimator determines the decay law and the penalty actually optimized ($t$: optimizer steps; $N$: rollouts; $k:=|2\bar q-1|\sqrt{G-1}$, $\omega\asymp k/G$). DAPO's per-step exponential and its rollout dilution cancel in $L(N)$ (Proposition~\ref{prop:conserve}); the bottom three rows optimize $\lambda_{\mathrm{eff}}=1$.}
\label{tab:menu}
\end{table}

\begin{figure}[t]\centering
\includegraphics[width=0.97\linewidth]{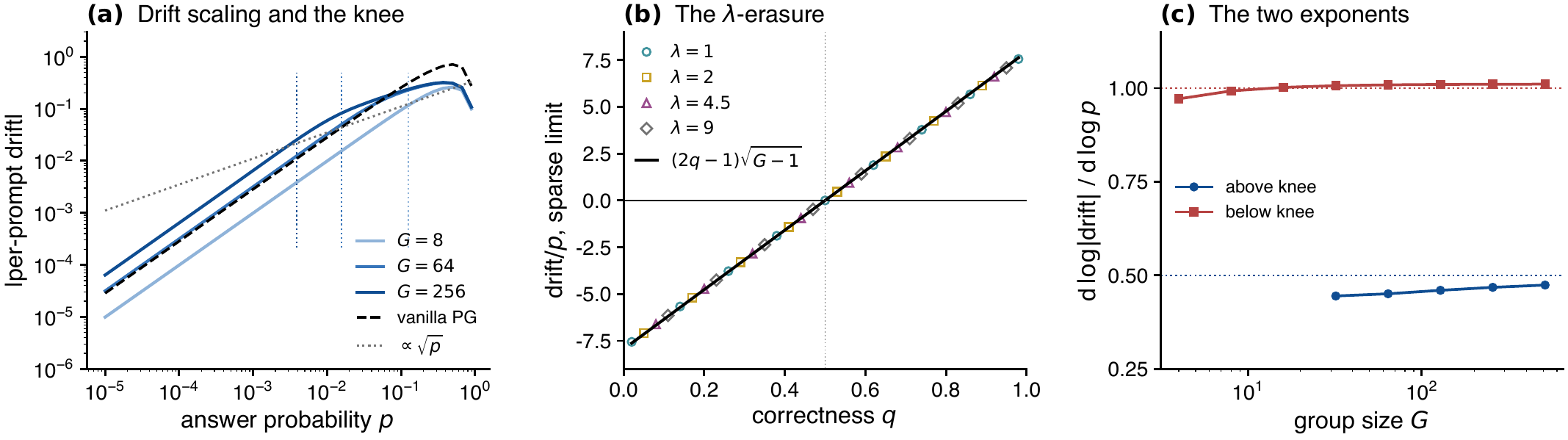}
\caption{The menu by exact enumeration ($\bar q=0.3$, $\lambda=4.5$). (a)~GRPO drift $\propto\sqrt p$ above the knee $p=1/G$ (dotted), $\propto p$ below; vanilla PG $\propto p$ throughout. (b)~Sparse regime: drift$/p$ collapses onto $(2\bar q-1)\sqrt{G-1}$ for every nominal $\lambda$. (c)~Exponents converge to $0.5$ and $1.0$.}
\label{fig:menu}
\end{figure}

\emph{The threshold moves.} In regime (c) the drift changes sign at $\bar q=1/2$, not at $t^\ast$: the optimizer trains the model to answer the entire band $\bar q\in(1/2,t^\ast)$ on which the rule assigns negative gain; at $\lambda=4.5$ this is the band $(0.50,0.82)$. Group normalization thus silently swaps in a rule whose Chow threshold is $1/2$, manufacturing the hallucinations the rule was introduced to prevent. \emph{The abstention rate is set by $G$, not $\lambda$.} On that band the two regimes push in opposite directions, creating a stable fixed point at $p^\ast=\Theta(1/G)$: the abstention rate on the contested band is set by the group size.

\begin{proposition}[Sampling filters conserve the collapse]\label{prop:conserve}
Let $D$ be the event that a group's reward vector is constant (e.g.\ all-abstain), on which $\hat A_i\equiv0$. Then $\E[\hat g]=\Pr[D^c]\,\E[\hat g\mid D^c]$ identically, so a filter that discards degenerate groups and resamples (DAPO; \citealp{yu2025dapo}) conserves the expected \emph{reward} drift per rollout exactly, for every $p,G,\bar q,\lambda$ and every prompt distribution, while taking $1/\Pr[D^c]\approx1/(Gp)$ times more rollouts per step. The anchor, applied once per optimizer step, is thereby \emph{diluted} by $\Pr[D^c]$ on the rollout axis. Per step the law becomes exponential, $L=\Theta(e^{-\omega t})$ with $\omega\asymp|2\bar q-1|/\sqrt G$; per rollout nothing changes. Consequently no resampling scheme that discards zero-gradient groups can repair the collapse: only changing the score or the advantage estimator can.
\end{proposition}

\paragraph{Falsifiable predictions.}
The section compresses to five estimator-indexed predictions: (P1) vanilla-PG and mean-baseline runs show local $L$-exponent $-1$; (P2) GRPO shows $-2$ above the knee; (P3) the knee sits at per-prompt answer rate ${\approx}1/G$ and moves with $G$; (P4) below the knee the drift is exactly $\lambda$-invariant, with amplitude $(2\bar q-1)\sqrt{G-1}$; (P5) GRPO and GRPO+DAPO trajectories coincide on the rollout axis at $\beta=0$; the experiments take them to a live model (Table~\ref{tab:menu} and Figure~\ref{fig:menu} display the enumerated menu).

\section{The Report-Level Repair}\label{sec:repair}

The repair removes abstention from the action space: answer every question, train a mandatory confidence report, threshold only at deployment.

\begin{proposition}[No shared vanishing factor; anchor-robust]\label{prop:report}
The population Brier objective satisfies $\partial_{\kappa_i}B=-2\E_p[(c-\bar q)\,\partial_{\kappa_i}c]$ with no action-saturation factor. At the calibrated optimum $c\equiv\bar q$ the Hessian in the range parameters is $-2\E[\nabla c\,\nabla c^\top]$, negative definite whenever $\phi$ is nonconstant on the support of $p$: an \emph{interior attracting} equilibrium ($\bar q$ bounded away from $0$ and $1$), within $O(\beta)$ of calibration under the proximal anchor, losing $O(\beta)$ deployment utility in general and $O(\beta^2)$ under the margin condition $\Pr_p(|\bar q-t^\ast|\le\epsilon)=O(\epsilon)$, which a single transversal crossing under a bounded signal density implies. The conclusion is anchor-robust: under a quadratic function-space anchor the optimum is $c_\beta=\bar q+\tfrac{\beta}{2+\beta}(c_0-\bar q)$; under a Bernoulli-KL anchor on the \emph{report}, the anchor's derivative diverges at the boundary and actively repels degenerate reports. At the action level, reward and anchor share the vanishing factor $\sgm'(v)$; at the report level they share none: the Brier gradient is $\Theta(1)$ away from calibration. With the feature frozen, the anchored objective in the range parameters is strongly concave: convergence is global and exponential, where the same reduction leaves the action-level rule collapsing.
\end{proposition}

\looseness=-1 The anchor asymmetry is structural, not elective: a deterministic report has no action distribution, hence no policy KL to anchor, and the proposition shows its conclusion is invariant to the anchor geometry that replaces it. On the action side the anchor is part of the finding. Under the true policy KL that RLHF uses, the action arm deploys nothing at any signal quality (Table~\ref{tab:abs}); replacing that KL by a weight-space proximal term, which production stacks do not use, rescues medium and strong signal in the model (coverage $0.234$ and $0.468$ against rational $0.259$ and $0.469$) but still forfeits all of weak signal ($0.000$ against $0.035$) and buys its sharpness with unbounded parameter norm (Figure~\ref{fig:mechanism}b). The report level needs no anchor surgery and is near-rational at all three signal qualities. One scope note on ``no vanishing factor'': with $c=\kappa_0+(\kappa_1-\kappa_0)\phi$, gradients in the feature parameters $(w,a)$ still carry $\phi'$ and can saturate. What the report removes is a shared factor multiplying \emph{every} coordinate. For any fixed feature the range parameters alone drive the deployed report to calibration at rate $\Theta(1)$, so the output has no reachable region of zero gradient. When $\bar q$ is not representable by the head, the same gradient drives $c$ to its constrained $L^2(p)$ projection, still a unique attractor, interior whenever no box constraint is active, and deployment pays at most $(1+\lambda)\,\E_p|c-\bar q|$ in utility for the approximation error. The experiments meet this as the binding readout constraint (in the live runs the full model trains under the composite score, the report read out by a trained linear head on the prompt state).

\begin{proposition}[The composite objective; pure properness is not enough]\label{prop:composite}
Score an answer (correct with probability $q$) and report $c$ by $r_\alpha=\alpha\,\ind{\mathrm{correct}}+1-(Y-c)^2$. Then $\partial_c\,\E[r_\alpha]=2(q-c)$ and $\partial_q\,\E[r_\alpha]=\alpha-1+2c$.
\textup{(a)} The report is driven to calibration for every $\alpha$. \textup{(b)} Pure Brier ($\alpha=0$) makes the calibrated accuracy gradient $2q-1<0$ for every $q<1/2$: the model gains score by \emph{degrading} accuracy on hard prompts and reporting the degradation honestly; these are the same prompts condition \textup{(B1)} selects. \textup{(c)} Any $\alpha>1$ makes the accuracy gradient strictly positive everywhere; $\alpha=1$ leaves a degenerate stationary point at $q=c=0$. The composite contains no discrete abstain action, every episode returns a $q$-correlated score, and the vertex analysis of Proposition~\ref{prop:vertex}\textup{(b)} is unchanged: no collapse is introduced.
\end{proposition}

Each term has one job (the correctness term drives accuracy; the proper term drives the report to calibration, which the deployment threshold reads). Under a correctness-only score the report never moves and deployment has nothing to threshold.

\begin{remark}[Properness is not the operative variable]\label{rem:clip}
Clip the strictly proper Brier rule's \emph{training} signal below $t^\ast$ (constant score, zero gradient there): the dead zone returns, the learned confidence slides onto the threshold, and deployment utility is weakly worse at every anchor strength tested, though the rule remains proper throughout, if no longer strictly. The operative variable is whether the training signal contains a reachable region of identically zero gradient. This also delimits when empirical successes of calibration-reward training \citep{damani2025rlcr} should be expected to transfer: properness must act on a continuous report channel, not a gated action.
\end{remark}

\noindent The design lesson, in one sentence: \emph{train confidence continuously, with a proper score plus a correctness reward at weight $\alpha>1$, and threshold only at deployment.}

\noindent The lesson names a family: any design that keeps answering mandatory in training and learns confidence outside the action channel, including post-hoc calibration \citep{guo2017} and conformal deployment-time abstention \citep{yadkori2024}, inherits the immunity; the composite is the member whose calibration and accuracy directions the propositions above certify jointly.

\section{Experiments}\label{sec:experiments}

\subsection{Exact and Simulated Confirmation}\label{sec:sim}

\emph{Mean-field enumeration.} All entries of Table~\ref{tab:menu} are verified by exact enumeration over group outcomes: the sparse-regime drift$/p$ matches $(2\bar q-1)\sqrt{G-1}$ to three decimals for every $\lambda\in\{1,2,4.5,9\}$, the two exponents converge to $0.5$ and $1.0$ as $G$ grows, and the rollout-axis conservation of Proposition~\ref{prop:conserve} holds to six significant figures.

\emph{Simulations.} \looseness=-1 In a two-type model calibrated favorably to the penalty rule ($q_H,q_L=0.9,0.2$; true gate KL for the action arm, greedy deployment, horizon $T=10^4$, $\beta=10^{-3}$): the collapse law holds (Figure~\ref{fig:onet}: $L$ falls onto one curve across two decades of $\beta$, below the $1/t$ envelope past the transient; the measured GRPO-minus-vanilla slope difference is $-1.013$ against a predicted $-1$); and the deployment outcome is one-sided (Table~\ref{tab:abs}): at strong signal the penalty rule forfeits \emph{all} of a $46.9\%$-coverage optimum while the report-level mechanism captures $97\%$ of its utility.

\begin{figure}[t]\centering
\includegraphics[width=0.97\linewidth]{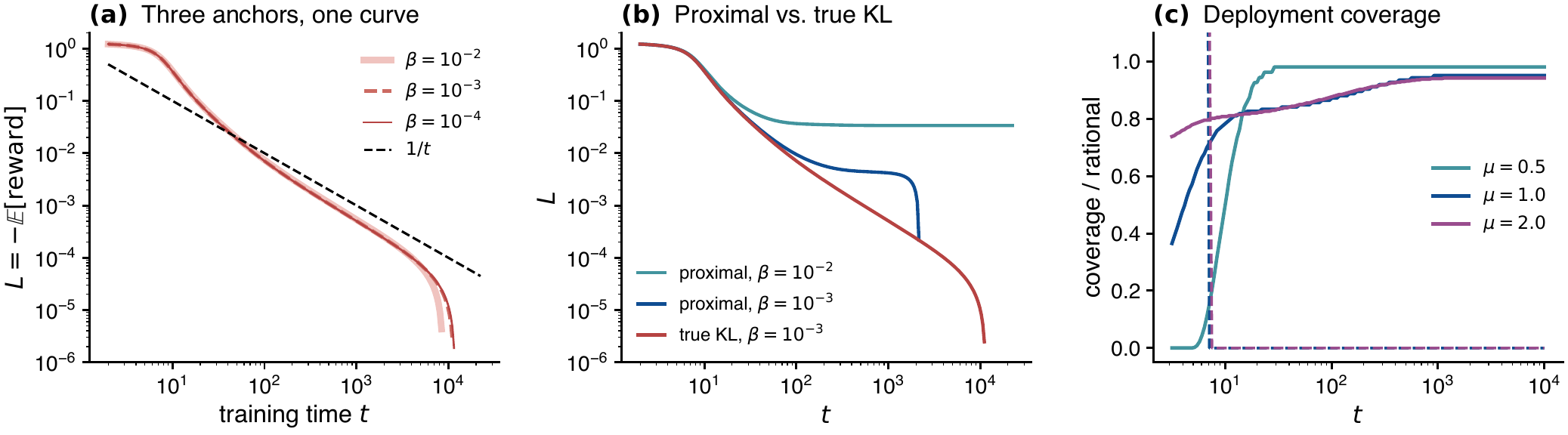}
\caption{(a)~The invariant $L$ for three anchor strengths; the upturn is $t_{\mathrm{tilt}}$. (b)~The proximal surrogate floors the reward; true KL does not. (c)~Greedy coverage at $\lambda=4.5$: report level (solid) near-rational; action level (dashed) deploys nothing.}
\label{fig:onet}
\end{figure}

\begin{table}[t]\centering\small
\setlength{\tabcolsep}{4.5pt}
\begin{tabular}{@{}lcc cc cc@{}}
\toprule
& \multicolumn{2}{c}{optimal} & \multicolumn{2}{c}{action-level} & \multicolumn{2}{c}{report-level}\\
\cmidrule(lr){2-3}\cmidrule(lr){4-5}\cmidrule(lr){6-7}
signal & $U^\ast$ & cov & $U$ & cov & $U$ & cov\\
\midrule
weak   & $0.005$ & $0.035$ & $\mathbf{0.000}$ & $\mathbf{0.000}$ & $0.005$ & $0.034$\\
medium & $0.074$ & $0.259$ & $\mathbf{0.000}$ & $\mathbf{0.000}$ & $0.074$ & $0.246$\\
strong & $0.199$ & $0.469$ & $\mathbf{0.000}$ & $\mathbf{0.000}$ & $0.194$ & $0.442$\\
\bottomrule
\end{tabular}
\caption{Deployment utility and coverage at $\lambda=4.5$ ($\beta=10^{-3}$, $T=10^4$). At $\lambda=1$ this calibration has $\E[q]=0.55>t^\ast$, \textup{(B1)} fails, and neither mechanism collapses.}
\label{tab:abs}
\end{table}

\emph{An empirical anchor.} For base policies with a sharp monotone confidence threshold, a \emph{negative} error-penalized score implies initial drift toward collapse: at least $33$ of more than $36$ public frontier models satisfy this at the deployed $\lambda=1$ \citep{aa2025omniscience}, and no single $\lambda$ is training-safe across models (derivation and scope in the appendix).

\subsection{Language-Model Experiments}\label{sec:llm}

We test the theory on Qwen2.5-1.5B (three seeds per run) and, at scale, Qwen2.5-7B (single seeds) \citep{qwen25}, on short-form QA from TriviaQA and PopQA \citep{joshi2017,mallen2023}, in two tiers. \emph{Tier 1 (theorem-grade, head-only):} the paper's four-parameter gate head is trained on a frozen feature of the base model (the answer/abstain logit difference), with rewards from pre-graded answer banks, so that $L$, $\bar\sgm$, and $\rho$ are computed \emph{exactly} per step; this tier tests the envelope of Theorem~\ref{thm:onet}(i), and sweeps the estimator menu (P1--P5) across $G\in\{64,256\}$ and $\lambda\in\{1,4,9\}$ with tolerances stated in the appendix. \emph{Tier 2 (full-parameter):} fine-tuning of the full model with the unbiased RLOO estimator. The abstain action is a binary first-token gate in the model's own vocabulary, with no additional head: the first generated position is restricted to a designated answer/refusal pair (refusal immediately ends the sequence), so the pair is the entire action space there, and the per-prompt answer probability is read exactly from it in one forward pass. The division of labor is deliberate: Tier 1 tests the estimator menu against exact enumerated references, and Tier 2 tests the mechanism with an estimator that cannot manufacture a collapse. Tier 2 compares action-level abstention training against report-level composite training ($\alpha=2$) on a two-tier prompt mixture satisfying (B1), with two controls: per-prompt capability certified at the capability checkpoints by forced-answer rollouts (Clopper--Pearson lower bound above $t^\ast$), so gate collapse is not confounded with capability loss; and an attribution control that zeroes the low-accuracy tier's reward-channel advantages, keeping its KL channel and paired randomness. A per-checkpoint gate-versus-content KL decomposition monitors the branch-length alternative.

\begin{figure}[t]\centering
\includegraphics[width=0.97\linewidth]{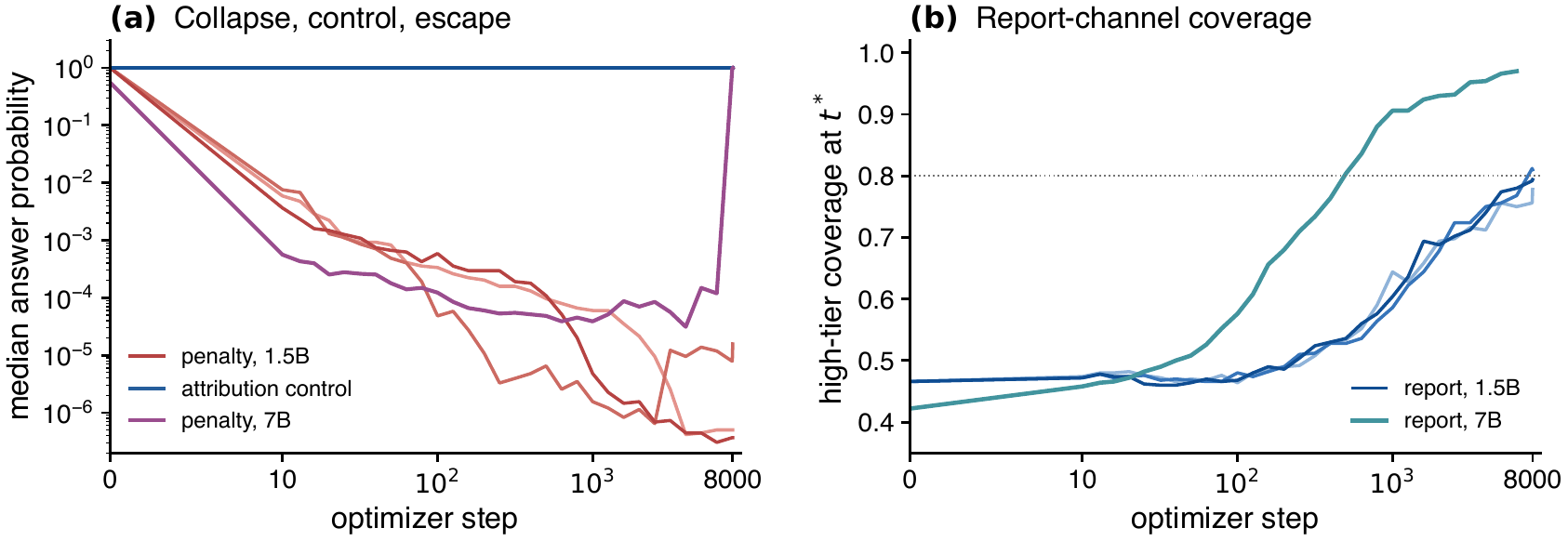}
\caption{Live-model dynamics on the capability-certified cohort. (a)~The penalty rule collapses the median answer probability within ten steps; the attribution control never moves. The 7B run collapses, plateaus, then escapes late, its paired control rising to full answering. (b)~Under the composite objective, coverage rises jointly with accuracy and calibration.}
\label{fig:live}
\end{figure}

\smallskip
\noindent\textbf{Tier-1 outcomes.} \looseness=-1 On the real pool every prediction lands within tolerance: the drift-menu local slopes are $0.507$ (group-normalized; predicted $\tfrac12$) and $0.951$ (mean-baseline; predicted $1$); the knee scales by $3.79$ for $G=64{\to}256$ (predicted ${\approx}4$), landing within $1.1\%$ of the enumerated reference; below the knee the drift is $\lambda$-invariant at every grid point; the first-window mean reward rises in every head-only configuration ($p\le7{\times}10^{-4}$); and the envelope of Theorem~\ref{thm:onet}(i) is falsified in no seed. Two comparisons are structurally unavailable on this pool (the sparse-regime collapse leaves no common eligible window for the trajectory-slope contrast; the bimodal base gate exhausts the dynamic-sampling budget from the first step). The full table is in the appendix.

\smallskip
\noindent\textbf{Tier-2 outcomes.} \looseness=-1 On a held-out $500$-prompt high-accuracy cohort, the median answer probability falls from $1.000$ to ${\le}0.008$ within ten optimizer steps and to ${\le}2{\times}10^{-5}$ by step $8000$ (exact readouts, not sampling estimates), in three of three seeds, while forced-answer correctness on those prompts holds at $0.94$--$0.97$ throughout ($256$ rollouts per prompt; one-sided $95\%$ Clopper--Pearson lower bounds above $t^\ast$ at every capability checkpoint for $461/433/432$ of the $500$). The attribution control isolates the cause: zeroing only the low-accuracy tier's reward-channel advantages, keeping the KL channel and the paired randomness identical, holds the cohort median at $1.000$ at every probe point, in all three paired seeds. At the collapsed terminal state, on a fixed diagnostic batch, the task-gradient norm is $0$ and the gate-KL-anchor gradient norm is below $2{\times}10^{-10}$, against a task-gradient norm of $108$ in the paired control at the same step: the reward gradient and the restoring force die together, and only where the collapse happened (Figure~\ref{fig:live}a).

\smallskip
\noindent\textbf{The report level, and two boundaries measured at scale.} \looseness=-1 Trained on the same mixture with the same estimator (three seeds), the report channel moves opposite to the gate: coverage $0.47\to0.78$--$0.81$, accuracy $0.43\to0.69$, Brier $0.27\to0.22$, jointly (Figure~\ref{fig:live}b), the direction the interior attracting optimum of Proposition~\ref{prop:composite} predicts. One caveat scopes this: at this scale the binding constraint is the readout, not the mechanism (a linear prompt-state probe opens at coverage $0.47$ on a tier whose true correctness exceeds $0.9$), so the 1.5B runs certify the direction of the report channel rather than near-rational deployed utility; the 7B run closes most of that gap. Single-seed runs at 7B sharpen two boundaries. First, the overconfident-base premise is a property of scale and template rather than of the mechanism: offered the abstain option, the 7B base answers only $0.51$ of its answerable questions (against $0.91$ without it), already partway toward Chow's rule. (B1) still holds on the training mixture and the saturation throttle is fully open at a half-answering start, so collapse remains the prediction, and it arrives: the 7B gate collapses by step $10$ and sits near $10^{-4}$ for roughly $6000$ steps, silencing even a base that already abstains on half of what it knows, while the paired control, started from the same gate, climbs to $0.93$ within the same ten steps and to full answering just after. The failed premise removes only the motivating overconfident scenario, not the collapse. Second, the escape is real, and late: the run leaves the plateau between steps $6.3$k and $7.9$k, recovering to median $1.0$ with the task gradient revived ($131$, against $0$ at the 1.5B terminal). The recovery is margin-selective (terminal median $1.0$, mean $0.65$), and the 1.5B horizon shows no escape: both regimes of the escape-time's exponential sensitivity (Proposition~\ref{prop:bounded}) are observed. The 7B report channel ends at coverage $0.98$, accuracy $0.94$, Brier $0.05$.

\section{Related Work}\label{sec:related}
\looseness=-1 Beyond the ternary-reward line already cited, RL-trained abstention includes advantage reweighting \citep{tiar2026}, knowledge-boundary, uncertainty-aware, and clarification rewards \citep{karl2026,ucpo2026,abstainr1_2026,xu2024rejection}, and confidence-shaped rewards \citep{zhang2025confclip,li2025klcrl,yang2025barrel}. Several report, as engineering obstacles, phenomena our theory derives: an ``abstention trap,'' advantage bias, vanishing gradients, refusal erosion \citep{song2025hallucination}. Reported ternary-reward successes \citep{truthrl2025,tiar2026} are reconciled with the theory in the limitations below. \citet{kalai2024calibrated} prove calibrated models must hallucinate and \citet{kalai2025why} prescribe error-penalized scoring; our results are the dynamical caveat to that program. Closest to our repair, \citet{damani2025rlcr} train correctness plus a Brier score on reported confidence. Remark~\ref{rem:clip} and Proposition~\ref{prop:composite} supply why the action-level alternative fails and why the correctness weight must exceed the properness weight. GRPO, its critiques, and its filters \citep{shao2024deepseekmath,liu2025drgrpo,ahmadian2024,yu2025dapo,hu2025reinforcepp} supply the estimator menu. Reward overoptimization \citep{gao2023,skalse2022} is orthogonal, since our score \emph{is} the deployment utility. Chow's rule and selective prediction \citep{chow1970,elyaniv2010,geifman2017,kamath2020,ren2023}, confidence elicitation and calibration \citep{kadavath2022,lin2022,tian2023,kumaran2026,guo2017,kuhn2023}, and conformal deployment-time abstention \citep{yadkori2024} support the deployment half of our design lesson. An extended discussion is in the appendix.

\section{Limitations and Conclusion}\label{sec:limitations}

\looseness=-1\textbf{Reconciling with reported ternary-reward successes.} TruthRL \citep{truthrl2025} and TIAR \citep{tiar2026} reduce hallucination without runaway refusal, and there is no contradiction. Both deploy $\lambda=1$, where the sparse-regime rewriting $\lambda\to\lambda_{\mathrm{eff}}=1$ is the identity and the over-answering band is empty. Collapse requires (B1) on the \emph{training} mix, which curated or retrieval-augmented data plausibly violate. And where (B1) holds, a rising mean reward at typical budgets \emph{is} the plateau's signature, not evidence against it. \textbf{Scope of the lifting.} Theorem~\ref{thm:lift}'s open cases remain open as theory, but not as observation: the Tier-2 runs sit in exactly this class, a fully shared decoder with a binary first-token gate in its own vocabulary, and exhibit the collapse and the joint gradient death. Free-form refusals, with no designated decision position, remain untested. Production estimator features (sampled-token variance, ratio clipping, the $\epsilon$ floor in the normalizer, length normalization) remain outside the analysis. Both live backbones are one model family (Qwen2.5) on short-form QA, with training mixtures constructed to satisfy (B1). \textbf{The report channel must be trainable toward correctness.} Verbalized confidence empirically tracks answer \emph{commitment} more than correctness \citep{kumaran2026}. The composite supervises the report against realized correctness, but our live runs read the report out through a linear probe on the prompt state: the verbalized channel remains untested. With an LLM judge crediting hedged text at rate $\nu$, properness degrades by $O(\nu)$: our claims are scoped to programmatic grading.

\smallskip
\noindent\textbf{Conclusion.} \looseness=-1 Error-penalized abstention is the right prescription for a rational agent and, we prove, a self-defeating one for a KL-anchored gradient learner whose prompts share a bounded readout, with the advantage estimator, not the designed rule, setting the collapse law's exponent, knee, and effective penalty. The remedy is not to stop penalizing errors but to move the penalty to the deployment rule. Scoring rules are designed for rational reporters; the geometry they induce on gradient learners is a separate design surface.

\bibliography{refs_abstention}

\clearpage
\appendix
\setcounter{secnumdepth}{2}

\noindent\emph{This appendix contains the extended related-work discussion, the complete proofs, the full lifting theorem, the complete estimator and report-level analyses, negative results, simulation details, the leaderboard derivation, and the language-model experiment design. Results restated from the main text are cited by their main-text numbers (``Theorem~1 of the main text''); results stated only in this appendix are numbered within their sections (e.g., Theorem~F.1).}
\medskip

\section{Extended Related Work}\label{app:related}
The main text's related-work section compresses the following discussion.

\paragraph{RL for abstention and hallucination.}
Recent work trains abstention with RL: ternary-reward GRPO \citep{truthrl2025}, advantage reweighting for abstention \citep{tiar2026}, knowledge-boundary-aware rewards \citep{karl2026}, uncertainty-decoupled advantages \citep{ucpo2026}, joint abstention and clarification rewards \citep{abstainr1_2026}, reliability rewards for unanswerable questions \citep{xu2025reliability}, RL from knowledge feedback \citep{xu2024rejection}, and confidence-shaped rewards \citep{zhang2025confclip,li2025klcrl,yang2025barrel}. Several of these report, as engineering obstacles, phenomena our theory derives: an ``abstention trap'' \citep{karl2026}, advantage bias under uncertainty rewards \citep{ucpo2026}, vanishing gradients from coarse rewards \citep{zhang2025confclip}, and refusal erosion under standard RL fine-tuning \citep{song2025hallucination}. We offer a common mechanism, rates, and conditions for when such training fails, together with a repair derived from the mechanism rather than patched onto it. \citet{truthrl2025} and \citet{tiar2026} report ternary-reward successes on real models; the main text's limitations reconcile this with the theory.

\paragraph{Vanishing gradients and RL fine-tuning dynamics.}
\citet{mei2020} give softmax policy-gradient rates governed by the optimal action's probability; \citet{razin2024} show reward fine-tuning stalls on inputs where the reward's standard deviation under the policy is small. Our contribution is not that $\sgm'$ can vanish, but the co-death identity (equation~(1) of the main text): the \emph{regularizer} believed to hold the policy near its base carries the same vanishing factor, an interaction absent from both analyses and specific to action-level abstention. Entropy collapse in RL for reasoning \citep{cui2025entropy} and RLHF's diversity loss \citep{kirk2024} are adjacent self-throttling phenomena on a different quantity (policy entropy); our collapse can occur with entropy intact.

\paragraph{Scoring rules and the statistical account of hallucination.}
\citet{kalai2024calibrated} prove calibrated models must hallucinate; \citet{kalai2025why} locate persistence in binary-graded incentives and prescribe error-penalized scoring. Our results are the dynamical caveat to that program: statically optimal scoring rules can be dynamically unreachable for a KL-anchored gradient learner. Closest to our repair, \citet{damani2025rlcr} train with correctness plus a Brier score on reported confidence and prove the objective's optimum is accurate and calibrated; our analysis supplies the missing half: \emph{why} the action-level alternative fails, why properness must act on a continuous report channel rather than a gated action (clipping the training signal below $t^\ast$ reinstates the failure; Remark~1 of the main text), and why the correctness weight must strictly exceed the properness weight (Proposition~9 of the main text).

\paragraph{KL-regularized RL, estimators, and reward hacking.}
KL regularization is ordinarily a stabilizer with provable averaging benefits \citep{geist2019,vieillard2020}; we exhibit a structural exception. GRPO \citep{shao2024deepseekmath}, its critiques \citep{liu2025drgrpo,ahmadian2024}, DAPO \citep{yu2025dapo}, and global normalization \citep{hu2025reinforcepp} supply the estimator menu; reward overoptimization \citep{gao2023,skalse2022} is an orthogonal failure channel (our score \emph{is} the deployment utility, so there is no proxy gap).

\paragraph{Selective prediction and calibration.}
The threshold $t^\ast$ is Chow's rule \citep{chow1970}; risk--coverage foundations are due to \citet{elyaniv2010,geifman2017}, with NLP instantiations in selective QA \citep{kamath2020,ren2023}. Confidence elicitation \citep{kadavath2022,lin2022,tian2023}, its commitment bias \citep{kumaran2026}, calibration \citep{guo2017,kuhn2023}, and conformal deployment-time abstention \citep{yadkori2024} together support the deployment half of our design lesson: the report channel is trainable, and thresholding belongs at inference.

\section{Notation, Conventions, and Observables}\label{app:notation}

Throughout, $\sgm(z)=1/(1+e^{-z})$, $\sgm'=\sgm(1-\sgm)$. Prompts carry correctness $q(x)$; the decision layer sees a signal $s\sim p(s)$ with posterior $\bar q(s)=\E[q\mid s]$; the penalty rule is $(+1,-\lambda,0)$ with margin $m(u)=(1+\lambda)u-\lambda$, threshold $t^\ast=\lambda/(1+\lambda)$, and gain density $g(s)=p(s)\,m(\bar q(s))$. The bounded feature is $\phi(s)=\sgm(ws+a)\in(0,1)$; the action-level gate is $v_\theta(s)=c_4\phi(s)+c_0$, $\theta=(w,a,c_4,c_0)$; the report head is $c_{\theta'}(s)=\kappa_0+(\kappa_1-\kappa_0)\phi(s)$, $\theta'=(w,a,\kappa_0,\kappa_1)$. The anchored objective is
\begin{align}
J_\beta(\theta)&=\int\sgm(v_\theta)\,g\;-\;\beta\,\mathcal K(\theta),\notag\\
\mathcal K(\theta)&:=\int p\,\KL\big(\Bern(\sgm(v_\theta))\,\|\,\Bern(\sgm(v_0))\big).
\label{eq:objS}
\end{align}
Standing condition \textbf{(B1)}: $D_0=-\int g>0$, equivalently $\E[q]<t^\ast$.

\paragraph{Observables.} $L(t):=-J(\theta_t)=\int\sgm(v)(-g)$ (minus the mean training reward); answer rate $P(t)=\int\sgm(v)\,p$; anchor share $\rho(t):=\beta\,\nabla J\cdot\nabla\mathcal K/\|\nabla J\|^2$; weighted gate mean $\bar\sgm(t):=\int\sgm^2(-g)\big/\!\int\sgm(-g)$ (the weight $\sgm(-g)$ is signed; no sign of $\bar\sgm$ is ever assumed). With $r(s):=c_4\phi(s)$: tilt $\Lambda(t)=\E_p[e^{r}]$; effective drift $D_{\mathrm{eff}}(t)=\int e^{r}(-g)$; tilted margin $\bar m_{\mathrm{tilt}}=\E_{\mathrm{tilt}}[m(\bar q)]$ under the measure $\propto e^{r}p$, so that $D_{\mathrm{eff}}=-\Lambda\,\bar m_{\mathrm{tilt}}$ exactly. $Z_L$ denotes the zero of $L$; $Z_m=t_{\mathrm{tilt}}$ the zero of $\bar m_{\mathrm{tilt}}$; $G_\pm:=\int g_\pm$; $\underline D(c_\star):=\min_{[0,c_\star t_{\mathrm{tilt}}]}D_{\mathrm{eff}}$; first moment $M_1:=\int|s||g|\,ds<\infty$. All measured values quoted below are taken at the working calibration (Section~\ref{app:sim}); ``the run'' means $(\mu,\lambda,\beta)=(1,4.5,10^{-3})$.

\paragraph{The parameter-space certificate.} Several derivations use the a-priori condition
\begin{equation}
\beta\,\|\theta_t-\theta_0\|_\infty\ \le\ \tfrac12\,D_{\mathrm{eff}}(t),
\qquad t\le c_\star t_{\mathrm{tilt}},\ c_\star<1,
\label{eq:plateauS}
\end{equation}
which certifies the observable hypothesis $\rho\le\tfrac12$ of Theorem~1 up to the $\Theta$ constant of $\nabla\mathcal K=\Theta(e^{c_0}\|\theta-\theta_0\|_\infty)$. The certificate is conservative by about two orders of magnitude: at $\beta=10^{-3}$ it lapses at $t=7.8\times10^3$ (against $t_{\mathrm{tilt}}=1.15\times10^4$), while the measured share $\rho$ peaks at $0.0080$ at $t\approx2\times10^3$, crosses zero near $t\approx7\times10^3$ as the reward and anchor gradients de-align, and reaches $-0.011$ just before $t_{\mathrm{tilt}}$: the anchor never spends more than $2\%$ of its allowed budget.

\section{Proofs of the Static Results}\label{app:static}

\paragraph{Threshold identity.} $g(s)=p(s)[(1+\lambda)\bar q(s)-\lambda]$ and $p>0$, so $g(s)>0\iff\bar q(s)>\lambda/(1+\lambda)$. Hence the penalty rule's optimal decision region and the calibrated-confidence threshold region coincide, and the two mechanisms are compared on the same target. \qed

\paragraph{Proof of Lemma 1 (shared throttling factor).} For $P=\Bern(\pi)$, $Q=\Bern(\pi_0)$, $\partial_\pi\KL=\mathrm{logit}\,\pi-\mathrm{logit}\,\pi_0$. With $\pi=\sgm(v)$, $\mathrm{logit}\,\pi=v$ and $d\pi/dv=\sgm'(v)$, so $\tfrac{d}{dv}\KL=\sgm'(v)(v-v_0)$; the gradient formula follows by differentiating \eqref{eq:objS} under the integral. \qed

\paragraph{Proof of Proposition 2 (the anchor cannot exclude the vertex).} Bernoulli KL to a fixed base is bounded: $\KL(\Bern(\sgm(v))\|\Bern(\sgm(v_0)))\le\max\{\log\frac1{\sgm(v_0)},\log\frac1{1-\sgm(v_0)}\}$ for every $v$. At the all-abstain vertex the cost is exactly $\overline{\KL}_\perp=\E_p[\log\frac1{1-\sgm(v_0)}]$, so $J_\beta(\text{all-abstain})-J_\beta(\text{base})=L(0)-\beta\overline{\KL}_\perp$, positive iff $\beta<\beta_{\mathrm{crit}}=L(0)/\overline{\KL}_\perp$. Since $\overline{\KL}_\perp\approx\log\frac1{1-p_{\mathrm{ans}}}$ grows only logarithmically in the base answer rate $p_{\mathrm{ans}}$, the threshold decays slowly: on the leaderboard-calibrated model of Section~\ref{app:empirical}, $\beta_{\mathrm{crit}}=0.139,\,0.116,\,0.089,\,0.060,\,0.041$ at $p_{\mathrm{ans}}=80\%,88\%,95\%,99\%,99.9\%$, above typical RLHF coefficients throughout (at $p_{\mathrm{ans}}=99.9\%$, $\beta=0.05$ the static argument alone no longer suffices and one must appeal to the dynamics). Two glosses. (i) The statement is that the anchor cannot \emph{exclude} the vertex from the reachable set; that the vertex is the \emph{destination} is the dynamical claim of Theorem~1. (ii) The comparison is not a dichotomy at the global level: under (B1) the abstain vertex (reward $0$) beats the base ($-D_0$), but the \emph{confidence} vertex $c\equiv0$ also beats an overconfident base under the Brier score. Only the local gradient at the vertex separates the mechanisms (Proposition~3). \qed

\paragraph{Proof of Proposition 3 (local dichotomy), with the tokenization caveat.} (a) $\partial_v\E[r]=\partial_v[\sgm(v)m(\bar q)]=\sgm'(v)m(\bar q)\to0$ as $v\to-\infty$, with $\sign=\sign m$. (b) By the Savage representation of a strictly proper score, $\E[S(c)]=\Gamma(c)+\Gamma'(c)(\bar q-c)$ with $\Gamma$ strictly convex, so $\partial_c\E[S]=\Gamma''(c)(\bar q-c)$; integrating ($\Gamma'$ strictly increasing) the expected score strictly increases from $c=0$ toward $\bar q$, so the vertex is never optimal and has no basin; if additionally $\Gamma''\ge \zeta>0$ near the boundary, the vertex gradient is $\ge \zeta\bar q>0$. The lower bound is not automatic for every proper score: $\Gamma(c)=c^4$ is strictly proper yet has $\partial_c\E[S]=12c^2(\bar q-c)=0$ at $c=0$, a first-order-degenerate (but still non-attracting) vertex. \emph{Caveat.} When the confidence is a sampled token, the vertex gradient in \emph{logit} space carries a softmax factor and vanishes degenerately; this is harmless in practice only because language models initialize overconfident, not underconfident \citep{guo2017,tian2023}, an empirical fact rather than a theorem. \qed

\paragraph{Proof of Corollary 1 (the regularizer class), with the floor.} $\partial_vR=(\partial\pi/\partial v)\,\partial_\pi R=\pi(1-\pi)\partial_\pi R\to0$ whenever $\pi\,\partial_\pi R\to0$. Membership: entropy has $\partial_\pi H=\log\frac{1-\pi}{\pi}$ and the forward anchor $\partial_\pi\KL=\mathrm{logit}\,\pi-\mathrm{logit}\,\pi_0$, both logarithmic, hence in the class; the reverse anchor $\KL(\pi_0\|\pi_\theta)$ has $\partial_vR=\pi-\pi_0\to-\pi_0\neq0$ and diverging vertex cost, hence outside; so is $R=-\log\pi$ ($\partial_vR\to-1$). The class fixes no rate: $\partial_\pi R=1/(\pi\log(1/\pi))$ qualifies yet gives $\partial_vR\asymp1/\log(1/\pi)$; entropy and forward KL themselves vanish like $\Theta(\pi\log(1/\pi))$, one logarithm slower than the reward's $\Theta(\pi)$, and that logarithm is the $|v-v_0|$ in the floor computation. (i) \emph{Entropy bonuses do not help}: $dH/dv=-v\,\sgm'(v)$ carries the same factor as the KL gradient and is of the same order on the plateau. (ii) \emph{The floor is exponentially small}: balancing $\sgm'(v)D_{\mathrm{eff}}$ against $\beta\sgm'(v)|v-v_0|$ gives $|c_0|\asymp D_{\mathrm{eff}}/\beta$ and $P_{\mathrm{floor}}\asymp e^{-D_{\mathrm{eff}}/\beta}$ (identically $e^{-D_{\mathrm{eff}}/\tau}$ for an entropy bonus of weight $\tau$), reached only after time $e^{\Theta(1/\beta)}$ in the frozen-readout bias dynamics; numerically, freezing the readout gain, the predicted stationary points $c_0=v_0-D_0/\beta$ (KL) and $c_0=-D_0/\tau$ (entropy) are reproduced to three decimals. The reverse KL floors at $P=\Theta(\beta)$ (balance $e^{v}|m|=\beta\pi_0$), one logarithm below the proximal floor of Proposition~4. \qed

\paragraph{Proof of Proposition 1 (tabular policies do not collapse).} With one logit per prompt the flow is $\dot v(s)=\sgm'(v(s))[m(\bar q(s))-\beta(v(s)-v_0(s))]$; since $\sgm'>0$, stationary points solve $m=\beta(v-v_0)$, i.e.\ $v^\ast(s)=v_0(s)+m(\bar q(s))/\beta$, and the bracket is decreasing in $v$, so $v^\ast$ is attracting. $\sign v^\ast=\sign m$ iff $m(m+\beta v_0)>0$, which fails only on a band of width $O(\beta)$ in $m$ around zero; a uniform margin $|m|\ge\gamma>\beta\|v_0\|_\infty$ removes it, and for every prompt with $m\neq0$ the decision converges to Chow's rule as $\beta\downarrow0$. Numerically, at $\beta=10^{-3}$ the tabular terminal policy agrees with Chow's rule on $100\%$ of the evaluation grid (the misclassified band, of width $\approx2\times10^{-3}$ in $m$, falls between grid points) and attains coverage $0.259$ against a rational $0.259$. \qed

\section{Proof of Theorem 1 (Collapse Law)}\label{app:onet}

\paragraph{(The identity.)} For a softmax policy and reward linear in the policy, $\partial J/\partial z_j=\pi_j(r_j-J)$ exactly. On a prompt with $m<0$, the abstain action is optimal ($r_{\mathrm{abs}}=0=J^\ast$), so with gap $\mathcal G=\sgm(v)|m|$, $|\partial_vJ|=(1-\sgm(v))\,\mathcal G$. In the population, $\partial_{c_0}J=\int\sgm'(v)g=\int\sgm(1-\sgm)g$, while $L=\int\sgm(-g)$; with $\bar\sgm$ as in Section~\ref{app:notation}, $\partial_{c_0}J=-(1-\bar\sgm)L$ \emph{exactly}. Measured $|\partial_{c_0}J|/L=0.971,\,0.995,\,0.9987,\,0.9997$ at $t=30,10^2,3\times10^2,10^3$. The exponent below is a property of the \emph{parameterization class}: under gradient flow in the logit (or any fixed linear reparameterization) it is $-1$; a nonlinear reparameterization can change it (flowing $\theta>0$ with $v=-\log\theta$ gives $\mathcal G\asymp t^{-1/3}$). Everything below is at the standard logit parameterization.

\paragraph{($\bar\sgm$ control.)} The weight $\sgm(-g)$ is signed, so $\bar\sgm\in(0,1)$ is not automatic (measured, $\bar\sgm$ touches $-0.0008$ at $t=6\times10^3$; near $Z_L$ it can diverge, which $t\le c_\star t_{\mathrm{tilt}}$ excludes). Since $|g|=-g+2g_+$,
\[
|\bar\sgm|\ \le\ \sup_s\sgm(v(s))\cdot\Big(1+\frac{2\int\sgm g_+}{L}\Big).
\]
Work in the run's gauge $c_4\le0$, so $e^{r}\le1$ and $\sup_s\sgm(v)\le e^{c_0}$. From $\sgm(v)\ge e^{v}(1-e^{\sup_sv})$, $L\ge e^{c_0}(D_{\mathrm{eff}}-e^{c_0}\!\int e^{r}g_-)\ge\tfrac12e^{c_0}D_{\mathrm{eff}}$ once $e^{c_0}\le\underline D(c_\star)/(2G_-)$, which the bias envelope below grants for $t\ge T_0(c_\star):=2C_bG_-/\underline D(c_\star)^2$ ($C_b:=2/(1-\bar\Lambda)$). Then $\int\sgm g_+/L\le2G_+/\underline D(c_\star)$, so on $T_0\le t\le c_\star t_{\mathrm{tilt}}$,
\[
|\bar\sgm(t)|\ \le\ \Big(1+\frac{4G_+}{\underline D(c_\star)}\Big)e^{c_0(t)}\ =\ O_{c_\star}(1/t).
\]
No sign of $\bar\sgm$ is needed; a negative excursion only speeds the certified decay.

\paragraph{(Gain growth: $|c_4|=O_{c_\star}(\log t)$, derived.)} Under \eqref{eq:plateauS} and the sign input $c_4\le0$ on the plateau (only the sign is used), $\dot c_4=\int\sgm'(v)\phi[g-\beta p(v-v_0)]$. Pointwise $|v-v_0|\le2\|\theta-\theta_0\|_\infty+|c_4(0)|$, so \eqref{eq:plateauS} gives $\beta\|v-v_0\|_\infty\le D_{\mathrm{eff}}+\beta|c_4(0)|$; with $\phi\le1$, $\sgm'\le e^{v}$, $e^{r}\le1$:
\[
|\dot c_4|\ \le\ \big(G_++2G_-+\beta|c_4(0)|\big)\,e^{c_0}.
\]
For the bias envelope, $\dot c_0\le e^{c_0}[-D_{\mathrm{eff}}+e^{c_0}G_-+(D_{\mathrm{eff}}+\beta|c_4(0)|)\Lambda]$; once the tilt input $\Lambda(\tau)\le\bar\Lambda<1$ holds past the transient (measured $\Lambda\le0.3$ from $t\ge10^2$; $\Lambda<1$ is automatic for $c_4<0$, a quantitative bound is not) and the small-$\beta$ condition
\begin{equation}
\beta\,|c_4(0)|\,\bar\Lambda\ \le\ \tfrac14\,(1-\bar\Lambda)\,\underline D(c_\star)
\label{eq:smallbeta}
\end{equation}
holds (a trajectory-level hypothesis \emph{not} implied by \eqref{eq:plateauS}; on the run it holds with factor $13$--$84$ to spare), $e^{-c_0}$ grows at rate $\gtrsim(1-\bar\Lambda)\underline D(c_\star)/2$ past a transient, whence $e^{c_0(\tau)}\le\min\{e^{c_0(0)},C_b/(\tau\underline D(c_\star))\}$. Integrating $|\dot c_4|$ against this envelope gives $|c_4(t)|\le|c_4(0)|+O_{c_\star}(\log t)$. The same bounds control the anchor gradient: $\beta|\partial_{c_0}\mathcal K|,\beta|\partial_{c_4}\mathcal K|\le(D_{\mathrm{eff}}+\beta|c_4(0)|)e^{c_0}=O_{c_\star}(L)$, while $\beta|\partial_w\mathcal K|,\beta|\partial_a\mathcal K|$ carry an extra $|c_4|$ and an $s$-factor bounded in first-moment form ($\phi'\le\tfrac14$; $\sup_s|\phi'(ws+a)s|$ can blow up as $w\to0$, $\E_p|s|$ cannot), so $\beta\|\nabla\mathcal K\|=O_{c_\star}((1+|c_4|)L)$. The certified constants are far from sharp; the measured fit is $|c_4|=0.469\log t-0.390$ with $R^2=0.9998$ (against $R^2=0.9739$ for the best power law).

\paragraph{(Upper bound and envelope.)} $\dot L=-\nabla J\cdot\nabla J_\beta=-\|\nabla J\|^2+\beta\nabla J\cdot\nabla\mathcal K$. Under the anchor-share hypothesis $\rho\le\tfrac12$ (for which \eqref{eq:plateauS} is the up-to-$\Theta$ certificate),
\[
\dot L\ \le\ -\tfrac12\|\nabla J\|^2\ \le\ -\tfrac12(\partial_{c_0}J)^2\ =\ -\tfrac12(1-\bar\sgm)^2L^2,
\]
so $1/L(t)\ge1/L(0)+\tfrac12\int_0^t(1-\bar\sgm)^2d\tau$: the envelope, needing nothing beyond $\rho\le\tfrac12$ (at $\beta=0$ the factor $\tfrac12$ disappears). The upgrade to $L=O(1/t)$ needs $\liminf_t\tfrac1t\int_0^t(1-\bar\sgm)^2>0$, supplied by $\bar\sgm$ control from \eqref{eq:plateauS} with the sign, tilt, and small-$\beta$ inputs. No pointwise $L\le((1-\bar\sgm)^2t)^{-1}$ follows (that would hold the time-varying $\bar\sgm$ at its smallest value; measured $Lt=1.26>1$ at $t=30$ is the counterexample); the finite-$t$ bound keeps the intercept $1/L(0)$.

\paragraph{(Lower bound.)} $\partial_{c_0}J,\partial_{c_4}J=O(L)$ ($\phi\in[0,1]$), but $\partial_wJ=\int\sgm'(v)c_4\phi'(ws+a)s\,g$ is bounded with $\phi'\le\tfrac14$ and $M_1<\infty$: $|\partial_wJ|\le\tfrac14|c_4|e^{c_0}M_1=O_{c_\star}(|c_4|L)$, likewise $\partial_aJ$. With $|c_4|=O(\log t)$ (gain growth) one gets $\|\nabla J\|^2=O(L^2\log^2t)$, and with $\beta\|\nabla\mathcal K\|=O_{c_\star}((1+|c_4|)L)$ the cross term is controlled in absolute value \emph{whatever its sign}: $|\dot L|\le\|\nabla J\|^2+\beta\|\nabla J\|\|\nabla\mathcal K\|=O(L^2\log^2t)$, so $L=\Omega(1/(t\log^2t))$. Both bounds require $t\le c_\star t_{\mathrm{tilt}}$, and every constant degrades as $c_\star\to1$ (the two lobes of $g$ cancel in $L$ at $t_{\mathrm{tilt}}$ but not in $\int\sgm'|g|$).

\paragraph{(Observables, and the endpoint.)} $D_{\mathrm{eff}}=-\Lambda\bar m_{\mathrm{tilt}}$ by definition; $L=e^{c_0}D_{\mathrm{eff}}(1+\varepsilon_L)$ and $P=\Lambda e^{c_0}(1+\varepsilon_P)$ with the \emph{exact} error functions
\[
\varepsilon_L=-\frac{\int e^{r}\sgm(v)(-g)}{D_{\mathrm{eff}}},\qquad
\varepsilon_P=-\frac{\int e^{r}\sgm(v)\,p}{\Lambda},
\]
satisfying $|\varepsilon_P|\le e^{\bar v}$ and $|\varepsilon_L|\le e^{\bar v}(1+2G_+/\underline D(c_\star))$ on the window ($\bar v:=\sup_sv=c_0$ in the gauge $c_4\le0$). Hence
\[
P=\frac{L}{|\bar m_{\mathrm{tilt}}|}\cdot\frac{1+\varepsilon_P}{1+\varepsilon_L}
\]
exactly away from $Z_L\cup Z_m$, and differentiating the defining integrals term by term gives $t(|\dot\varepsilon_L|+|\dot\varepsilon_P|)=O(e^{\bar v}\log t)=o(1)$: the slope identity
$\mathrm{d}\log P/\mathrm{d}\log t=\mathrm{d}\log L/\mathrm{d}\log t-\mathrm{d}\log|\bar m_{\mathrm{tilt}}|/\mathrm{d}\log t+o(1)$, which the run reproduces to three decimals in every instrumented window. Because $\Lambda\in[t^{-\Theta(1)},t^{\Theta(1)}]$ under $|c_4|=\Theta(\log t)$, the aggregate exponent of $P$ is \emph{not} determined by the theorem while that of $L$ is.

\emph{Transversality, reduced to a checkable condition.} $D_{\mathrm{eff}}$ does not depend on $c_0$; for $i\in\{w,a,c_4\}$, $\partial_iD_{\mathrm{eff}}=-\int(\partial_iv)e^{r}g=:-A_i$, while the flow moves $\dot\theta_i=e^{c_0}(A_i+e_i)$ with explicit plateau-approximation errors $e_i$. The chain rule gives
\[
\dot D_{\mathrm{eff}}\;=\;-\,e^{c_0}\big(A_w^2+A_a^2+A_{c_4}^2\big)\;-\;e^{c_0}\sum_iA_ie_i:
\]
minus a sum of squares. Consequently (i) $D_{\mathrm{eff}}$ is monotone decreasing on the plateau up to the stated errors (verified: exact $\dot D_{\mathrm{eff}}<0$ on all of $[30,1.2\times10^4]$, matching $-e^{c_0}\sum A_i^2$ to $0.99$--$1.01$); (ii) at a zero $Z$ of $D_{\mathrm{eff}}$, single-crossing of $g$ in $\phi$ (automatic for strictly monotone $\bar q$ and $w\neq0$) gives $|A_{c_4}(Z)|=\int e^{r}|\phi-\phi^\ast||g|=:\delta(Z)>0$ (measured $\delta(Z_m)=0.0299$ to five decimals). Provided the \emph{error-dominance condition} $\sum_i|A_ie_i|\le\tfrac12\sum_iA_i^2$ holds near $Z$, a smallness condition that must be checked separately (it is \emph{not} implied by \eqref{eq:plateauS}, whose right side vanishes at $Z_m$; measured ratio $1.01$), the zero is transversal with rate carrying its own $e^{c_0}$, and combining with the exact identity $|D_{\mathrm{eff}}(Z_L)|\le e^{c_0}\int e^{2r}|g|$ at $Z_L$, the two $e^{c_0}$ factors cancel:
\[
|Z_L-Z_m|\ \le\ \frac{2\int e^{2r}|g|}{\delta(Z_m)^2}\,\big(1+o(1)\big)\;=\;O(1)
\]
in \emph{absolute} time units. Measured: first-order prediction $|D_{\mathrm{eff}}(Z_L)|/|\dot D_{\mathrm{eff}}(Z_m)|=11.5$ time units against a measured gap $Z_L-Z_m=11.5$ ($Z_m=11452.3$, $Z_L=11463.8$: $0.10\%$). \qed

\paragraph{The four zeros, and which are observable.} The bias obeys $\dot c_0=e^{c_0}\Lambda\bar m_{\mathrm{tilt}}+\beta e^{c_0}\Theta(|c_0|)$, the anchor correction positive. Distinguish: $Z_L$ (zero of $L$; directly observable, no gauge); $Z_m$ (zero of $\bar m_{\mathrm{tilt}}$; observable by stratifying on confidence; $0.10\%$ from $Z_L$ on the run); $Z_A$ (zero of $\dot c_0$; \emph{not} observable, since $c_0$ is a coordinate of our parameterization); $Z_C$ (zero of $\mathrm{d}\log P/\mathrm{d}\log t$). The anchor makes $Z_A<Z_m$, with measured gaps $Z_m-Z_A=2285,\,981,\,355,\,96,\,17$ at $\beta=10^{-2},3\times10^{-3},10^{-3},3\times10^{-4},10^{-4}$; the gaps are \emph{not} proportional to $\beta$ (ratios vary by a factor of two), so no proportionality is claimed. $Z_L$ itself moves with $\beta$: $8745,\,10709,\,11464,\,11756,\,11842$ over the same grid: a $26\%$ shift at $\beta=10^{-2}$ but $3.2\%$ between $10^{-3}$ and $10^{-4}$. The defensible statement is: \emph{the mean reward crosses zero at a time that stabilizes as $\beta\downarrow0$, to within $5\%$ for $\beta\le10^{-3}$.}

\paragraph{Gauge-dependence, and the coincidence of zeros.} The decomposition of the $P$-slope into a bias term and a tilt term is gauge-dependent (the reparameterization of Section~\ref{app:bounded} moves mass between $c_0$ and $c_4$ while fixing $v(\cdot)$ pointwise); the observables $L$, $P$, $\bar m_{\mathrm{tilt}}$ are not, so the observed ``compensation'' between the two terms needs no proof. The falsifiable endpoint prediction is a \emph{coincidence of zeros}, not the value of any exponent: over successive plateau windows the aggregate slope $\mathrm d\log P/\mathrm d\log t$ runs $-1.016,\,-0.838,\,-0.618,\,-0.270,\,+0.068$ while $\bar m_{\mathrm{tilt}}$ runs $-1.318\to+0.013$. The slope flattens and flips exactly as $\bar m_{\mathrm{tilt}}$ does, and it is the coincidence of the two zero crossings that is the prediction.

\paragraph{Where the plateau ends, mechanically.} The tilt $e^{c_4\phi}$ up-weights the region of $\phi$ selected by the sign of $c_4$, and the trained head orders $\phi$ so that this is the \emph{profitable} region (on the run, $w<0$ and $c_4<0$: the mean of $\phi$ is $0.282$ on $g_+$ against $0.915$ on $g_-$, and $e^{-|c_4|\phi}$ up-weights exactly the former). Hence $\bar m_{\mathrm{tilt}}$ increases with $|c_4|$ and crosses zero at a value that is a constant of the problem; because $|c_4|$ grows like $\log t$, $t_{\mathrm{tilt}}$ is finite and, in the limit $\beta\downarrow0$, $\beta$-independent. The model escapes by \emph{re-weighting} its answering toward the profitable prompts, which is also why the escape is invisible to greedy deployment until much later (Section~\ref{app:bounded}).

\paragraph{A certified slope band with no unknown constants.} Put $R(t):=\|\nabla J\|^2/(\partial_{c_0}J)^2\ge1$. Exactly,
\[
-R\,(1-\bar\sgm)^2\,tL\ \le\ \frac{\mathrm{d}\log L}{\mathrm{d}\log t}\ \le\ -\tfrac12(1-\bar\sgm)^2\,tL,
\]
the left edge an edge only where $\rho\ge0$. All three factors are measurable; at $t=10^2,3\times10^2,10^3,3\times10^3$ the lower edges are $-1.315,-1.132,-1.095,-1.239$ against exact slopes $-1.310,-1.125,-1.086,-1.229$. The gap between slope and edge, in units of the edge, \emph{is} the anchor share $\rho(t)$ (identically), measured at $0.4$--$0.8\%$ across the windows and matching the independently computed $\rho$ to five decimals: this is the one place the anchor's effect is measured rather than bounded. The band \emph{widens} like $\log^2t$: $R-1=\Theta(c_4^2)$, which is the lower bound's polylog in observable form, an independent and fit-free prediction.

\section{Bounded Readouts and the Two Escapes}\label{app:bounded}

\paragraph{Proof of Proposition 5.} $\sup_sv(s)=\max(c_4,0)+c_0$ since $\phi$ ranges over $(0,1)$; greedy decoding answers a positive-measure set iff this exceeds $0$. The scalar is gauge-invariant: the reparameterization $(w,a,c_4,c_0)\mapsto(-w,-a,-c_4,c_0+c_4)$ leaves $v(\cdot)$ pointwise unchanged while flipping the sign of $c_4$. For an unbounded feature $v=us+b$ with $u>0$ and full-support $s$, greedy answers the tail $\{s>-b/u\}$ for every finite $b$. For the tilt: escape requires $\int e^{r}g_+$ to outgrow $\int e^{r}g_-$; with $r=c_4\phi$, $\phi\in[0,1]$, $e^{r}\le e^{|c_4|}$ uniformly, so under gain growth ($|c_4|=O_{c_\star}(\log t)$) the available tilt is polynomial in $t$ and reaching gain $M$ takes time $e^{\Omega(M)}$; for a Gaussian feature the tilt $e^{u^2/2}$ is super-polynomial at the same gain growth. \qed

\paragraph{Two escapes.} The \emph{sampling} policy escapes at $t_{\mathrm{tilt}}$, where $D_{\mathrm{eff}}$ crosses zero; \emph{greedy} coverage recovers only when $\sup_sv(s)$ crosses zero, strictly later (Figure~\ref{sfig:bounded}). On the run $c_4<0$ throughout ($-1.78$ at $t=10^2$, $-19.5$ at $t=10^5$), so greedy recovery waits for $c_0$ itself to turn positive, between $t=3\times10^4$ and $5\times10^4$ at the working point, \emph{after} the sampling escape at $t_{\mathrm{tilt}}\approx1.15\times10^4$. Exit times are erratic in the signal quality because they depend on how far past $t_{\mathrm{tilt}}$ the recovery of $c_0$ must run: at $\beta=10^{-3}$ the first time greedy coverage is positive is $\le10$ for $\lambda\le3.5$ at every $\mu$; at $\lambda=4$ it is $2.1\times10^4$ ($\mu{=}1$), $7.4\times10^3$ ($\mu{=}2$), ${>}10^7$ ($\mu{=}0.5$); at $\lambda=4.5$, $3.3\times10^4$ ($\mu{=}1$) and ${>}10^7$ for $\mu\in\{0.5,2\}$. Exit times vary by less than a factor $2$ across $\beta\in[10^{-4},10^{-2}]$, as the asymptotic $\beta$-independence of $t_{\mathrm{tilt}}$ predicts; the non-monotonicity in $\mu$ is real (a ridge in the exit-time landscape, Figure~\ref{sfig:exit}), spanning more than five orders of magnitude across a band of $\lambda$ only two units wide. \emph{The plateau is a slow manifold, not a trap; its length relative to any realistic training budget, not absorption, is the claim.}

\begin{figure}[t]\centering
\includegraphics[width=0.95\linewidth]{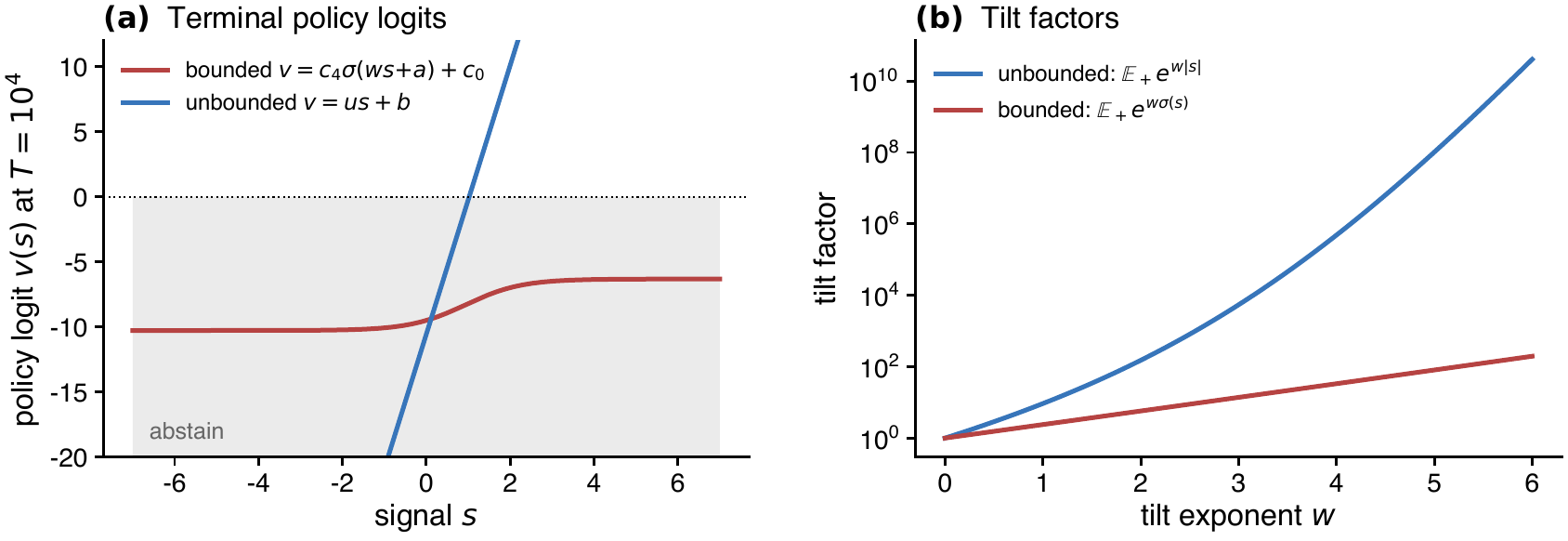}
\caption{(a)~Terminal policy logits at $T=10^4$, $\mu=1$, $\lambda=4.5$: with a bounded readout $v(s)<0$ for every $s$ (greedy abstains on everything); with an unbounded feature the logit crosses zero and greedy answers a tail. (b)~The tilt factor required for escape diverges for unbounded features and is capped for bounded ones.}
\label{sfig:bounded}
\end{figure}

\begin{figure}[t]\centering
\includegraphics[width=0.97\linewidth]{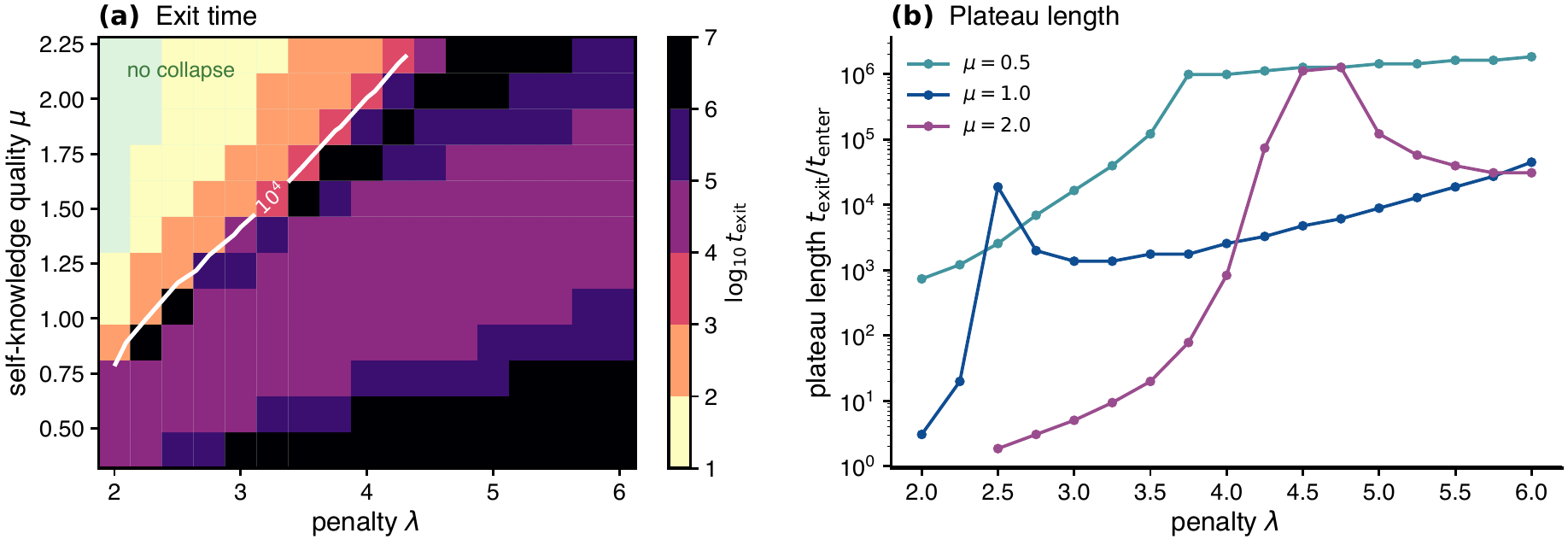}
\caption{A horizon-free view: $t_{\mathrm{enter}}$ is the first time greedy coverage falls below $2\%$ of rational, $t_{\mathrm{exit}}$ the first time thereafter it exceeds $20\%$. (a)~$\log_{10}t_{\mathrm{exit}}$ under true KL with a bounded head at $\beta=10^{-3}$; green cells never collapse. (b)~Plateau length $t_{\mathrm{exit}}/t_{\mathrm{enter}}$ explodes past a critical $\lambda$ that increases with signal quality $\mu$, non-monotonically; the ridge is displayed rather than smoothed.}
\label{sfig:exit}
\end{figure}

\section{The Lifting Theorem in Full}\label{app:lift}

\begin{stheorem}[Gate-to-sequence lifting; full statement]\label{sthm:lift}
Let $\pi_\theta(y\mid x)$ be a sequence policy, $A$ the refusal event, $\pi_{\mathrm{ans}}(x)=1-\pi_\theta(A\mid x)$, and anchor the objective with the full sequence KL. Write $K_{\mathrm{ans}}(x)=\KL(\pi_\theta(\cdot\mid x,A^c)\|\pi_0(\cdot\mid x,A^c))$, $K_{\mathrm{abs}}(x)$ likewise for the refusal branch, and $q_\theta(x)$ for the answer-conditional correctness.

\textbf{Part I (exact identities; any mixture parameterization).} For every policy,
\[
\KL_{\mathrm{seq}}(x)=\KL_{\mathrm{gate}}(x)+\pi_{\mathrm{ans}}K_{\mathrm{ans}}+\pi_{\mathrm{abs}}K_{\mathrm{abs}},
\]
where $\KL_{\mathrm{gate}}=\KL(\Bern(\pi_{\mathrm{ans}})\|\Bern(\pi^0_{\mathrm{ans}}))$.
If the gate has a logit $v$ on which the branch conditionals do not depend, then
\[
\partial_vJ_\beta=\sgm'(v)\big[m(q_\theta)-\beta(v-v_0)-\beta(K_{\mathrm{ans}}-K_{\mathrm{abs}})\big]:
\]
every term carries $\sgm'(v)$; Lemma~1, Proposition~2 (the all-abstain state with refusals at base costs exactly $\overline{\KL}_\perp$ under the sequence KL, so $\beta_{\mathrm{crit}}$ is unchanged) and Corollary~1 lift verbatim with $m_{\mathrm{eff}}=m(q_\theta)-\beta(K_{\mathrm{ans}}-K_{\mathrm{abs}})$. The anchor taxes only the branch taken: once answer content has drifted ($K_{\mathrm{ans}}>K_{\mathrm{abs}}$), its gate component pushes \emph{toward} abstention: abstaining hides content drift from the KL. (In the closed model this shift is second-order: $Z_L$ moves by $<0.1$ time units when the $K_{\mathrm{ans}}$ term is switched off; it becomes first-order when content drifts exogenously.)

\textbf{Part II (branch-separable content).} Assume $\psi=(\psi_{\mathrm{ans}},\psi_{\mathrm{abs}})$ with $q_\theta,K_{\mathrm{ans}}$ depending only on $\psi_{\mathrm{ans}}$ and $K_{\mathrm{abs}}$ only on $\psi_{\mathrm{abs}}$; refusals initialized at base; the answer branch within bounded KL of base, $\overline K_0:=\sup_xK_{\mathrm{ans}}(x,0)<\infty$; and the regularity bound $C_K:=\sup_{x,t}(|\partial_{\psi_{\mathrm{ans}}}q_\theta|+\|\nabla_{\psi_{\mathrm{ans}}}K_{\mathrm{ans}}\|)<\infty$, an explicit assumption (it holds in the Bernoulli content model, whose correctness probabilities stay in a compact subinterval of $(0,1)$; it does \emph{not} follow from a bound on the scalar correctness). Then the refusal branch is invariant ($K_{\mathrm{abs}}\equiv0$ along the flow), the collapsed manifold $\{\pi_{\mathrm{ans}}=0,\ \psi_{\mathrm{abs}}=\psi_{\mathrm{abs},0}\}$ is stationary in the limit, and the value term is controlled along the flow by the path bound
\[
\sup_xK_{\mathrm{ans}}(x,t)\ \le\ \overline K_0+(1+\lambda+\beta)\,C_K^2\!\int_0^tP(\tau)\,d\tau .
\]
The only force surviving off the manifold restores \emph{refusals} to base, never answering. With content \emph{shared} across branches this fails: $-\beta\pi_{\mathrm{abs}}\nabla_\psi K_{\mathrm{abs}}$ is $O(\beta)$ and unthrottled ($\pi_{\mathrm{abs}}\to1$); that case is open.

\textbf{Part III (finite-time sandwich, lifted).} Assume additionally $\rho_{\mathrm{full}}:=\beta\,\nabla J\cdot\nabla\KL_{\mathrm{seq}}/\|\nabla J\|^2\le\tfrac12$ on the full gradient; a capability ceiling $q_\theta\le q_{\max}$ with ceiling-\textup{(B1)} ($\underline D^{\max}(c_\star):=\min_{\mathrm{window}}\int e^{r}(-g_{\max})>0$ for $g_{\max}:=p[(1+\lambda)q_{\max}-\lambda]$); and the sign, tilt, and small-$\beta$ inputs of Theorem~1. Then the envelope and the $O(1/t)$ upper bound survive verbatim, and the certified lower bound weakens by one logarithm, $L=\Omega(1/(t\log^3t))$, with constants carrying $C_K$ and $\overline K_0$; the window hypothesis $\sup_{x,\tau\le c_\star t_{\mathrm{tilt}}}K_{\mathrm{ans}}\le\overline K<\infty$ (automatic in the Bernoulli content model) restores $\log^2t$. Without a ceiling, collapse persists whenever the race inequality
\[
\Delta_{\mathrm{B1}}\;>\;\big[(1+\lambda)\,C_q(T)+\beta\,C_{\nabla K}(T)\big]\int_0^{T}\!P(\tau)\,d\tau
\]
holds at $T=Z_L$, where $\Delta_{\mathrm{B1}}$ is the \textup{(B1)}-breaking parameter distance and $C_q,C_{\nabla K}$ split $C_K$ by force. The criterion is a \emph{logged-data certificate}, stated in the run's own parameterization and norm (the constants are not reparameterization-invariant): it certifies a given run, not a function-space property.
\end{stheorem}

\begin{proof}
\emph{(Part I)} The identity is the chain rule for KL over the partition $\{A,A^c\}$, plus $\partial_v[\sgm(v)K_{\mathrm{ans}}+(1-\sgm(v))K_{\mathrm{abs}}]=\sgm'(v)(K_{\mathrm{ans}}-K_{\mathrm{abs}})$, the conditionals being $v$-independent in a mixture parameterization. At the all-abstain state with refusals at base, $\pi_{\mathrm{ans}}=0$ and $K_{\mathrm{abs}}=0$, so $\KL_{\mathrm{seq}}(x)=\log\frac1{1-\sgm(v_0(x))}$ pointwise.

\emph{(Part II)} Under separability, $\nabla_{\psi_{\mathrm{abs}}}$ of the reward and of $K_{\mathrm{ans}}$ vanish identically, so $\dot\psi_{\mathrm{abs}}=-\beta\E_x[\pi_{\mathrm{abs}}\nabla_{\psi_{\mathrm{abs}}}K_{\mathrm{abs}}]$, zero at base where every $K_{\mathrm{abs}}(x)$ is minimized: $\psi_{\mathrm{abs}}$ is invariant. On the content channels the $\pi_{\mathrm{ans}}\to0$ limit follows by dominated convergence (integrands bounded by $C_K$, $\sup|m'|$). The gate channel carries the \emph{value} term $-\beta\sgm'(v)K_{\mathrm{ans}}$, which $C_K$ does not control along an unbounded trajectory; the flow structure does: $\dot\psi_{\mathrm{ans}}=\E_x[\pi_{\mathrm{ans}}((1+\lambda)\partial_\psi q-\beta\nabla_\psi K_{\mathrm{ans}})]$ gives $\|\dot\psi_{\mathrm{ans}}\|\le(1+\lambda+\beta)C_KP(t)$, whence the displayed path bound, and $\E_x[\sgm'(v)K_{\mathrm{ans}}]\le P(t)\sup_xK_{\mathrm{ans}}(x,t)$, which on any window with $P=O_{c_\star}(1/t)$ is $O(\log t/t)\to0$; $\E_x[\pi_{\mathrm{ans}}K_{\mathrm{ans}}]$ obeys the same bound, so the collapse path's sequence-anchor cost remains the gate's $\overline{\KL}_\perp$ alone. (Measured, on the trainable-content run: $\overline K_0=5.6\times10^{-4}$; $\sup_xK_{\mathrm{ans}}(x,Z_L)=1.6\times10^{-3}$ against $2.3\times10^{-3}$ from the measured displacement and $5.4\times10^{-3}$ fully certified; the gate value term peaks at $7.5\times10^{-5}$ and is $3.3\times10^{-8}$ at $Z_L$.) With shared $\psi$ the first step fails: $\nabla_\psi K_{\mathrm{abs}}\neq0$ once answer-side updates have moved $\psi$.

\emph{(Part III)} Pointwise $q_\theta\le q_{\max}$ with $m$ increasing gives $D_{\mathrm{eff}}(q_\theta)\ge\underline D^{\max}(c_\star)$ on the window, and $m_{\mathrm{eff}}\le m(q_{\max})$ ($-\beta K_{\mathrm{ans}}\le0$ dropped, $K_{\mathrm{abs}}\equiv0$), so the bias envelope, $\bar\sgm$ control, and gain growth run as in Section~\ref{app:onet} with $\underline D^{\max}$-type constants. \emph{Envelope:} under $\rho_{\mathrm{full}}\le\tfrac12$ (the hypothesis must sit on the full gradient because the content channels contribute their own cross term $\beta\,\partial_\psi J\cdot\partial_\psi\KL_{\mathrm{seq}}$, not sign-definite), $\dot L\le-\tfrac12(1-\bar\sgm)^2L^2$, the identity $\partial_{c_0}J=-(1-\bar\sgm)L$ holding at each instant for the current $g$. \emph{Lower bound:} past $T_0$, $P=\frac{L}{|\bar m_{\mathrm{tilt}}|}\cdot\frac{1+\varepsilon_P}{1+\varepsilon_L}\le\frac{2L}{\underline D^{\max}(c_\star)}$, so $|\partial_\psi J|\le(1+\lambda)C_KP=O_{c_\star}(L)$ and the anchor's content component is $O_{c_\star}(\beta L)$. The anchor's gate coordinates carry, beyond $\sgm'(v)(v-v_0)$, the value term $\sgm'(v)(K_{\mathrm{ans}}-K_{\mathrm{abs}})$: on $c_0,c_4$ ($|\partial v|\le1$) it is $O_{c_\star}(\beta L\log t)$ by the path bound; on $w,a$, $\partial_wv=c_4\phi'(ws+a)s$ carries the extra $|c_4|$, giving $O_{c_\star}(\beta L\log^2t)$; these coordinates are $c_4\phi'$-weighted, not $\phi$-weighted, so no a fortiori argument applies. Hence $\|\nabla J\|=O_{c_\star}(L\log t)$ but $\beta\|\nabla\KL_{\mathrm{seq}}\|=O_{c_\star}(L\log^2t)$, so $|\dot L|=O_{c_\star}(L^2\log^3t)$ and $L=\Omega(1/(t\log^3t))$; bounded $K_{\mathrm{ans}}$ on the window makes every value term $O(|c_4|L)$ and restores $\log^2t$. The total content signal is $\int|\partial_\psi J|=O_{c_\star}(\log t_{\mathrm{tilt}})$, which yields the race criterion: $\|\dot\psi_{\mathrm{ans}}\|\le[(1+\lambda)C_q+\beta C_{\nabla K}]P(t)$ pointwise, so if the displayed inequality holds at $T=Z_L$, (B1) holds on $[0,Z_L]$ and the collapse completes first.
\end{proof}

\paragraph{Numerical instantiation.} Simulated with a trainable correctness shift $q_u=q_0+(q_{\max}-q_0)\sgm(u)$ and minimal Bernoulli content KL: under the ceiling ($\E[q_{\max}]=0.65<t^\ast$) the collapse law is unchanged, with $Lt=0.71,\,0.51$ at $t=10^2,10^3$ against the frozen-content $0.71,\,0.51$, and content moves $\E[q]$ by $10^{-4}$ during the plateau; without any ceiling ($q_{\max}=0.95$, (B1) breakable) collapse still wins: (B1) first fails at $t=2.9\times10^4$, a factor $2.8$ \emph{after} $Z_L=1.0\times10^4$, the model having moved $\E[q]$ by $0.005$ of the needed $0.27$. Race certificate on that run: $\int_0^{Z_L}P=12.4$; window constants $C_q=0.023$, $C_{\nabla K}=0.0031$ give right side $1.54$ against $\Delta_{\mathrm{B1}}=4.71$ (actual movement $0.56$), so the run is certified; the cruder global constants give $12.8$ and fail to certify despite the conclusion holding: the certificate's sharpness depends on the Lipschitz constants supplied. \emph{The plateau starves the very channel that could avert it.}

\section{Estimator Analysis in Full}\label{app:estimator}

\paragraph{Derivation of the drift menu (Proposition 6).} GRPO forms $\hat A_i=(r_i-\hat\mu_r)/\hat\sigma_r$ within groups of $G$; since $\sum_i\hat A_i=0$, the per-prompt expected gate drift is $\E[S]/G$ with $S=\sum_{\mathrm{answering}}\hat A_i$, which we evaluate by exact enumeration over group outcomes $(k,j)$ ($k$ answers, $j$ correct among them, probabilities $\binom Gkp^k(1-p)^{G-k}\binom kj\bar q^j(1-\bar q)^{k-j}$, population-variance normalizer with denominator $G$ and $\epsilon=0$; degenerate groups $\hat A\equiv0$). (a) For vanilla PG and mean-baseline estimators the drift is $\Theta(p\,m)$ at every $p$ (mean-baseline: exactly $p\,m\,(1-p)(G-1)/G$, which the enumeration reproduces to $10^{-15}$). (b) For $1/G\ll p\ll1$, the sample statistics concentrate: $\hat\mu_r\approx pm$, $\hat\sigma_r\approx\sqrt{pM_2}$ with $M_2=\bar q+(1-\bar q)\lambda^2$, so drift $\approx m\sqrt{p/M_2}$; the enumerated log-log slopes approach $0.5$ from below as $G$ grows ($0.451/0.468/0.479$ at $G=64/256/1024$), the deficit being the $O(1/G)$ statistic-correlation bias. (c) For $p\ll1/G$, a group has at most one answer with probability $\approx Gp$; for the single-answer group $r=(r,0,\dots,0)$: $\hat\mu=r/G$, $\hat\sigma^2=r^2(G-1)/G^2$, so $\hat A_1=\sign(r)\sqrt{G-1}$ \emph{independently of $|r|$}. Hence drift $\to p\,(2\bar q-1)\sqrt{G-1}$: the penalty magnitude is erased, $\lambda_{\mathrm{eff}}=1$. Enumerated drift$/p$ matches $(2\bar q-1)\sqrt{G-1}$ to three decimals for every $\lambda\in\{1,2,4.5,9\}$.

\paragraph{Rate saturation in $\lambda$ (above the knee).} $m/\sqrt{M_2}\to-\sqrt{1-\bar q}$ as $\lambda\to\infty$: enumeration at $\bar q=0.3$, $G=256$, $p=0.05$ gives drift$/\sqrt p=-0.378,-0.582,-0.705,-0.751,-0.788,-0.793$ for $\lambda=1,2,4.5,9,50,200$ (limit $-0.837$; Figure~\ref{fig:saturation}), while vanilla drift$/p=m$ grows linearly from $-0.40$ to $-34.7$. The sign is preserved ($\sign(m/\sqrt{M_2})=\sign m$): regime (b) distorts the rate, not the threshold.

\begin{figure}[t]\centering
\includegraphics[width=0.8\linewidth]{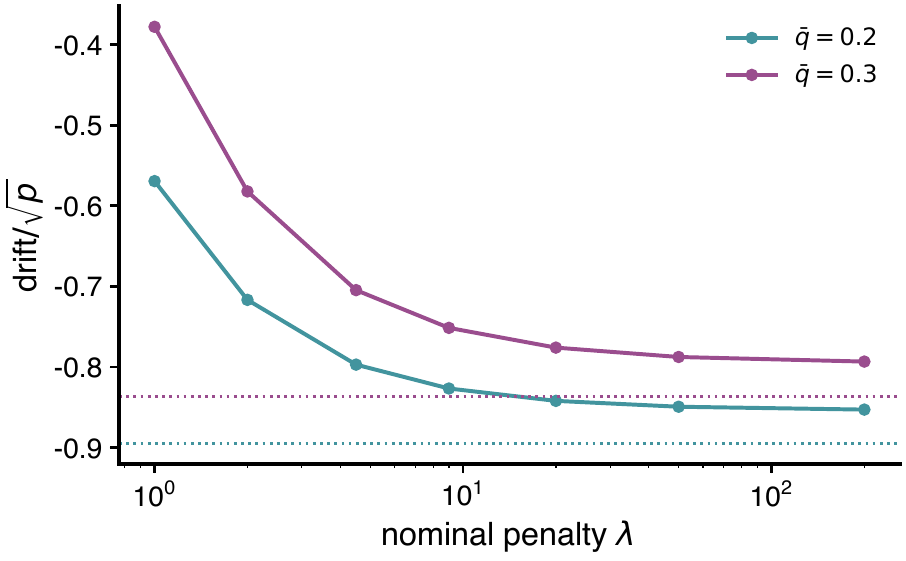}
\caption{Rate saturation above the knee: the population-normalized amplitude drift$/\sqrt{p}$ (enumerated at $p=0.05$, $G=256$) saturates at $-\sqrt{1-\bar q}$ (dotted) as the nominal $\lambda$ grows. Past the knee, making the penalty harsher stops making the collapse faster.}
\label{fig:saturation}
\end{figure}

\paragraph{Threshold shift and the $1/G$ fixed point (below the knee).} The regime-(c) drift $\propto(2\bar q-1)$ changes sign at $\bar q=1/2$, not at $t^\ast$ (Figure~\ref{sfig:tiers}): the optimizer answers the entire band $\bar q\in(1/2,t^\ast)$ on which the rule assigns negative gain (at $\lambda=4.5$: $(0.5,0.818)$). On that band regimes (b) and (c) push in opposite directions, so the drift crosses zero from above at a stable $p^\ast=\Theta(1/G)$ (constant increasing in $\bar q$); a tabular policy under GRPO therefore lands on a three-tier terminal profile ($0$ for $\bar q<1/2$; $\Theta(1/G)$ on the band; $1$ above $t^\ast$): \emph{the abstention rate on the contested band is set by the group size, not by the penalty.}

\paragraph{Proof of Proposition 7 (conservation), and its scope.} Let $D$ be the degenerate event (constant group reward). Tautologically $\E[\hat g]=\E[\hat g\ind{D^c}]=\Pr[D^c]\,\E[\hat g\mid D^c]$. GRPO spends $G$ rollouts per step and moves by $\E[\hat g]$; the resampling filter spends $G/\Pr[D^c]$ rollouts per step and moves by $\E[\hat g\mid D^c]$: the expected \emph{reward} drift per rollout is $\E[\hat g]/G$ for both, for every $p,G,\bar q,\lambda$ and every prompt distribution, with the acceptance weight cancelling the per-group rollout cost prompt by prompt. Hence the reward-only ($\beta=0$) mean-field expected trajectories coincide pointwise in rollout count $N$ (confirmed to six significant figures: $P=2.007716\times10^{-3}$ at $N=10^4$ under both). The anchor does \emph{not} cancel: it is applied once per optimizer step, and the filter takes $\Pr[D^c]\approx Gp$ as many steps per rollout, diluting the anchor's restoring force by that factor (at $G=8$, $p=10^{-3}$: $125\times$). Scope: the statement covers filters that discard \emph{zero-gradient} groups; filters that discard groups with nonzero gradient (variance thresholds, curricula) change the drift and are outside it, as are clip-higher and token-level losses (objective modifications). Per optimizer step the sparse-regime law becomes exponential, $\dot c_0\to(2\bar q-1)\sqrt{G-1}/G$, i.e.\ $L=\Theta(e^{-\omega t})$, $\omega\asymp|2\bar q-1|/\sqrt G$; per rollout, nothing changes (Figure~\ref{sfig:conserve}). \qed

\begin{sproposition}[The anchor under a sampling filter: a barrier, not a floor]\label{sprop:barrier}
In the homogeneous scalar model (one effective prompt, logit $c_0$, $p=\sgm(c_0)\approx e^{c_0}$ in the sparse regime, knee at $c_0^{\mathrm{knee}}=-\log G$), the filtered flow is
\[
\dot c_0=-\frac kG+\beta\,e^{c_0}(v_0-c_0),\qquad k=|2\bar q-1|\sqrt{G-1}.
\]
Assume $v_0+\log G>1$ (equivalently $c_0^{\mathrm{knee}}<v_0-1$), so the anchor term is strictly increasing on the sparse regime and largest at the knee. Comparing anchor and drive there gives
\[
\beta^{\mathrm{filter}}_{\mathrm{crit}}=\frac{k}{\,v_0+\log G\,},
\]
with an all-or-nothing structure (every sparse-regime root is repelling since $\mathrm{d}\dot c_0/\mathrm{d}c_0>0$ there): for $\beta>\beta^{\mathrm{filter}}_{\mathrm{crit}}$ a trajectory descending from above meets $\dot c_0>0$ at the knee and, provided the above-knee field also points down at the knee (a two-sided sign condition we check by enumeration rather than prove), stalls at $p\approx1/G$, an \emph{observable} answer rate; for $\beta<\beta^{\mathrm{filter}}_{\mathrm{crit}}$ it passes the knee and no stationary point exists thereafter. Numerically $\beta^{\mathrm{filter}}_{\mathrm{crit}}=0.259$ at $G=8$ and $0.516$ at $G=64$ ($\bar q=0.3$, $v_0=2$; the ratio $1.99$ matches the $\sqrt G/\log G$ scaling); realistic $\beta\le10^{-2}$ sit $1.4$--$1.7$ orders below, so collapse proceeds. The anchor is not weakened into a lower floor; it is converted into a switch.
\end{sproposition}

\begin{sremark}[Two different $1/G$'s]
The barrier above sits at the same answer rate as the drift fixed point $p^\ast=\Theta(1/G)$, and the mechanisms are unrelated: one is the anchor balancing a constant drive, the other the reward drift changing sign. They separate experimentally by varying $\beta$, which moves the barrier and leaves $p^\ast$ fixed; any experiment that finds a plateau at $p\approx1/G$ must run this control before attributing it. This is also the one row of the estimator menu whose behavior is \emph{not} $\beta$-independent.
\end{sremark}

\paragraph{Zero gradients are exact, not small.} When every sample in a group abstains, all $G$ rewards equal $0$, $\hat\sigma_r=0$, and $\hat A_i\equiv0$: the prompt's contribution is identically zero. The fraction of prompts in this state is $(1-p)^G$: at $p=10^{-3}$, $99.2\%$ for $G=8$ and $77.4\%$ for $G=256$.

\paragraph{From drift to decay law.} With $L\asymp e^{c_0}D_{\mathrm{eff}}$ and $D_{\mathrm{eff}}$ slowly varying, $\dot L/L=\dot c_0(1+o(1))$, so the menu's drift scalings convert to laws in one line each:
\begin{align*}
\dot c_0\propto-e^{c_0}&\Rightarrow L\sim t^{-1};\qquad
\dot c_0\propto-e^{c_0/2}\Rightarrow L\sim t^{-2};\\
\dot c_0\to-c\ &\Rightarrow L\sim e^{-ct}.
\end{align*}
(GRPO is not a gradient flow on $J$, so Theorem~1's envelope does not apply to it; the transfer above is what does.)

\paragraph{Two practical notes on dynamic sampling.} (i) The published method also removes the KL term entirely \citep{yu2025dapo}, so the anchor-dilution channel is moot in practice; the conservation of the reward drift is the operative content. (ii) The resampling budget is capped in practice: once $Gp$ falls below the cap the batch cannot be filled and training stalls, so the real curve tracks the unfiltered estimator on the rollout axis and then halts.

\paragraph{Group normalization and the report channel.} The report-level mechanism survives group-based estimation for a reason worth separating from determinism: within a prompt all $G$ samples answer, so normalization rescales a strictly proper score by a prompt-level statistic. For \emph{fixed} scale this preserves properness \citep{gneiting2007}; but $\hat\mu_r,\hat\sigma_r$ are sample statistics correlated with $r_i$ (exactly the bias identified by \citet{liu2025drgrpo}), so under group-standard-deviation normalization properness is preserved only up to an $O(1/G)$ bias, and \emph{exactly} under mean-only baselines (RLOO, Dr.~GRPO). This gives the design lesson a second leg: \emph{move abstention to the report level, and drop the standard-deviation normalization.} Action-level rules enjoy no such invariance at all, because the abstain baseline $0$ and the penalty $-\lambda$ are coupled through the action distribution, which is the content of this section.

\begin{figure}[t]\centering
\includegraphics[width=0.97\linewidth]{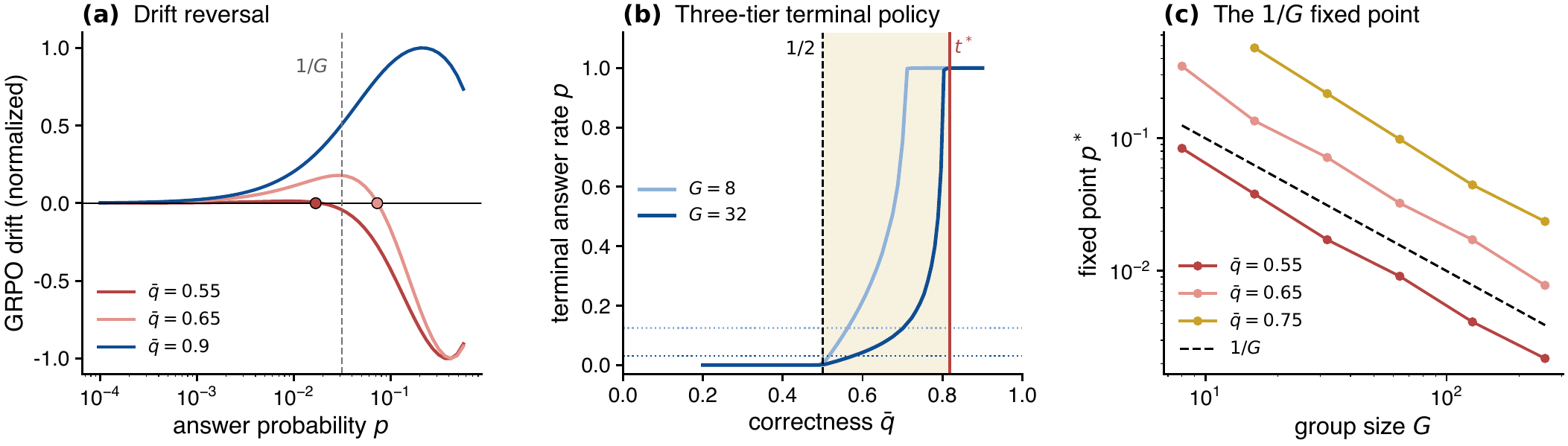}
\caption{The fixed point that group normalization creates. (a)~For $\bar q$ inside the band $(1/2,t^\ast)$ the GRPO drift is negative above $p\approx1/G$ and positive below it: a stable interior fixed point (circles); for $\bar q>t^\ast$ the drift is positive everywhere. (b)~Terminal profile of a tabular policy under GRPO: three tiers, with the answering threshold at $\bar q=1/2$ rather than the rule's $t^\ast=0.818$; the shaded band is where the rule assigns negative gain and the optimizer answers anyway. (c)~The fixed point scales as $\Theta(1/G)$.}
\label{sfig:tiers}
\end{figure}

\begin{figure}[t]\centering
\includegraphics[width=0.97\linewidth]{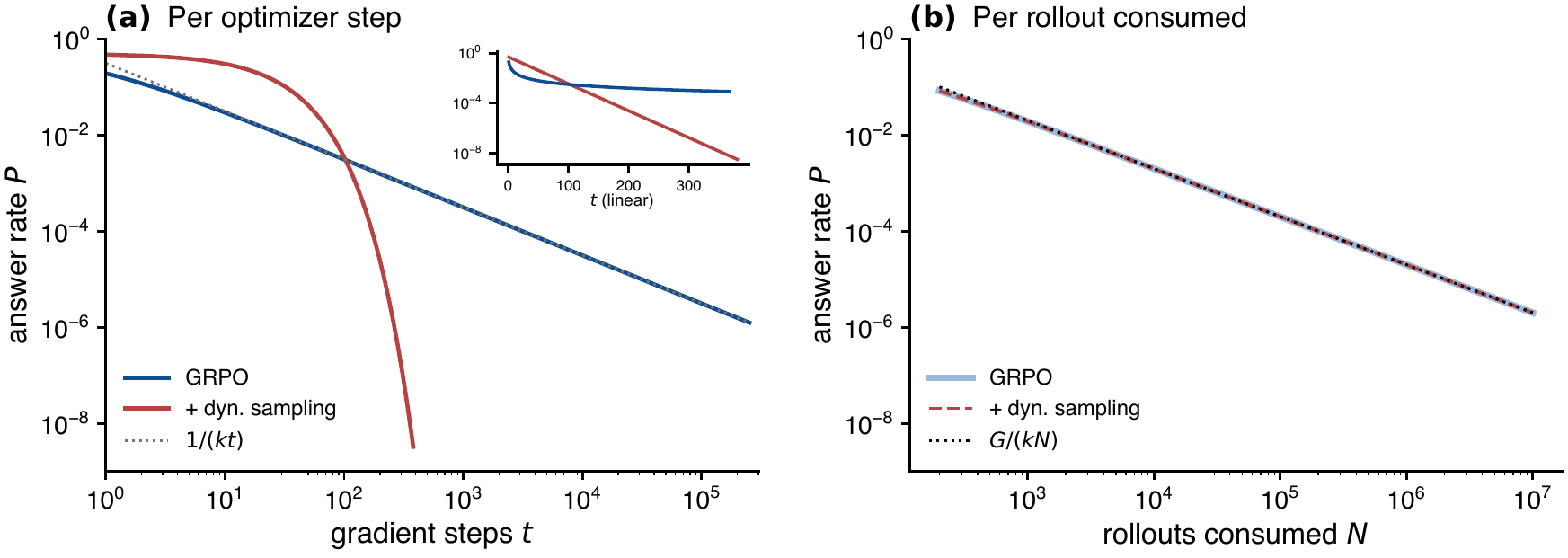}
\caption{Conservation of collapse per rollout. (a)~Against optimizer steps the two estimators obey different laws: unfiltered on log--log, filtered on semi-log (inset). (b)~Against rollouts consumed, the two curves coincide with each other and with $G/(kN)$ to six significant figures (reward-only, $\beta=0$; with the anchor on, the fields differ by the $\Pr[D^c]$-diluted restoring term). The paired test (separation in steps, coincidence in samples) requires no curve fitting, and no plausible confound produces both.}
\label{sfig:conserve}
\end{figure}

\section{The Report-Level Mechanism in Full}\label{app:report}

\paragraph{Boundary versus interior.} Under action-level abstention two information channels die together at the vertex: the action channel's per-episode Fisher information about the gate logit is $I_v=\E_p[\sgm'(v)]=\Theta(P(\mathrm{answer}))\to0$, and the correctness channel is throttled by $P(\mathrm{answer})$ because the abstain branch returns a constant (for the feature parameters $(w,a)$ the Fisher matrix carries additional Jacobian factors, so the honest statement is $I(w,a)=O(P(\mathrm{answer}))$ absent a nondegeneracy floor). Under report-level scoring every episode returns a $q$-correlated score, and the driving term $\partial_wB=-2\E_p[(c-\bar q)(\kappa_1-\kappa_0)\sgm'(ws+a)s]$ loses its discriminative factor only if the reported range collapses ($\kappa_1\to\kappa_0$), an \emph{interior} point of parameter space where the negative-definite Hessian below supplies restoring force, versus the \emph{boundary} point $P\to0$ where, by Corollary~1, every policy-space restoring force in its class vanishes. Boundary versus interior is the exact content of ``kills.'' (The two sides are analogous, not identical, objects: a per-episode Fisher information about the gate logit on one side, the curvature of the population objective in the report's range parameters on the other; a unified comparison would fix one observation likelihood for both mechanisms, which we have not done. The invariant content is the boundary-versus-interior distinction; Figure~\ref{sfig:info}a shows the two quantities falling as one curve.)

\begin{figure}[t]\centering
\includegraphics[width=0.97\linewidth]{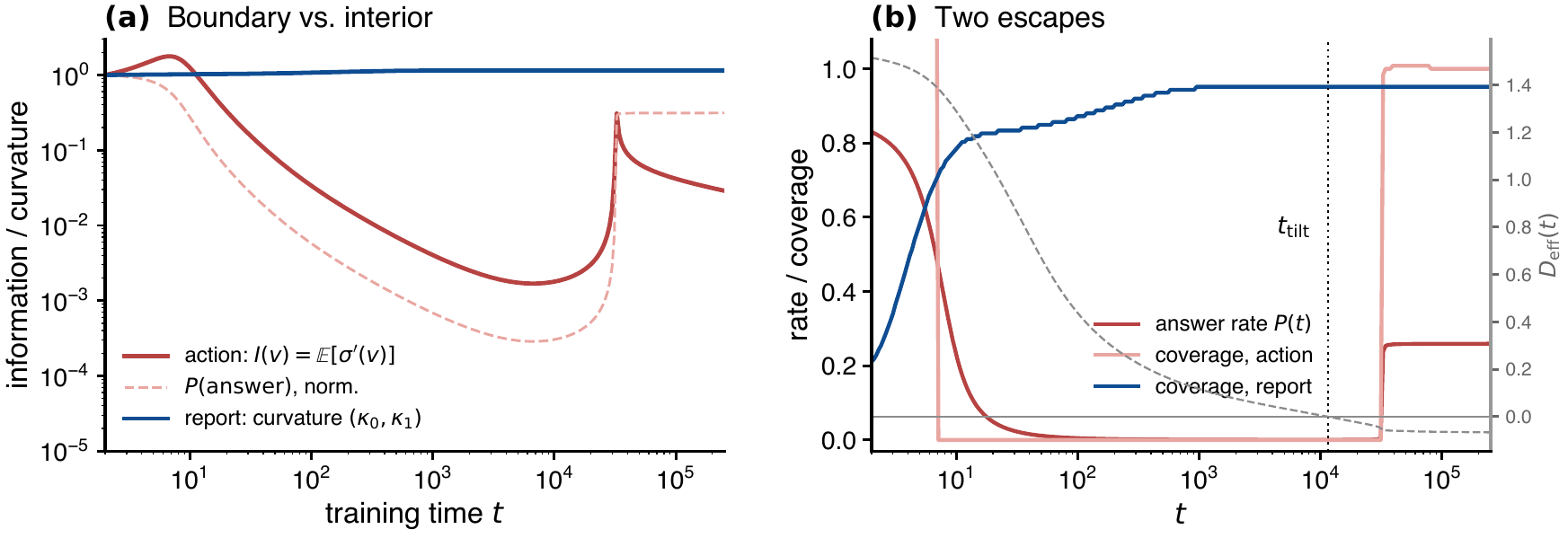}
\caption{(a)~Boundary versus interior: the action-level rule's per-episode information $\E_p[\sgm'(v)]$ tracks $P(\mathrm{answer})$ to zero, while the report-level objective's curvature $\E_p[\phi^2+(1-\phi)^2]$ is bounded below; its degenerate direction is interior, where restoring force is strictly positive. (b)~Two escapes: the sampled rate turns up at $t_{\mathrm{tilt}}$, where $D_{\mathrm{eff}}$ (gray, right axis) crosses zero; greedy coverage recovers strictly later; the report-level mechanism never leaves near-rational coverage.}
\label{sfig:info}
\end{figure}

\paragraph{Proof of Proposition 8 (report equilibrium; anchors).} $\E[(Y-c)^2\mid s]=(c-\bar q)^2+\bar q(1-\bar q)$, so $B(\theta')=\mathrm{const}-\E_p[(c_{\theta'}-\bar q)^2]$, with $\partial_{\kappa_0}B=-2\E_p[(c-\bar q)(1-\phi)]$, $\partial_{\kappa_1}B=-2\E_p[(c-\bar q)\phi]$: no action-saturation factor. At the realizing parameter (which exists in the working model: $\bar q(s)=q_L+(q_H-q_L)\sgm(2\mu s)$, realized by $(w,a,\kappa_0,\kappa_1)=(2\mu,0,q_L,q_H)$) the Hessian is $-2\E[\nabla c\,\nabla c^\top]\prec0$ (components of $\nabla c$ linearly independent in $L^2(p)$), and the implicit function theorem applied to $\nabla B(\theta')-\beta(\theta'-\theta'_0)=0$ yields an attracting anchored equilibrium within $O(\beta)$; since $\bar q$ is strictly increasing with $\bar q'(s^\ast)>0$ and $g(s^\ast)=0$, an $O(\beta)$ perturbation of $c$ moves the deployment threshold by $O(\beta)$ and costs $O(\beta^2)$ utility (the operative condition is the margin bound $\Pr_p(|\bar q-t^\ast|\le\epsilon)=O(\epsilon)$, implied here by the single transversal crossing, $\bar q'(s^\ast)>0$, with $p$ bounded near it: the anchored solutions above satisfy $\|c-\bar q\|_\infty=O(\beta)$ uniformly, so disagreement requires $|\bar q-t^\ast|=O(\beta)$, a set of mass $O(\beta)$ on which the per-point gap $(1+\lambda)|\bar q-t^\ast|$ is itself $O(\beta)$; without the margin bound, the generic $(1+\lambda)\E_p|c-\bar q|=O(\beta)$ stands). \emph{Function-space anchors.} (a) Quadratic $\tfrac\beta2\!\int p(c-c_0)^2$: pointwise maximization gives $c_\beta=\bar q+\tfrac{\beta}{2+\beta}(c_0-\bar q)$. (b) Bernoulli KL between reported confidences: $h(c)=2(\bar q-c)-\beta(\mathrm{logit}\,c-\mathrm{logit}\,c_0)$ is strictly decreasing with $h\to\pm\infty$ at the endpoints, so a unique root $c_\beta$ strictly between $\bar q$ and $c_0$, and $|c_\beta-\bar q|\le\tfrac\beta2\Lambda_0$ with $\Lambda_0=\sup_s|\mathrm{logit}\,\bar q-\mathrm{logit}\,c_0|<\infty$ for $\bar q$ bounded away from $\{0,1\}$ (numerically $\|c_\beta-\bar q\|_\infty=1.08\beta$ at $\beta=10^{-2}$, inside the bound $2.2\beta$). In both cases the same $O(\beta^2)$ deployment loss under the same margin condition. \emph{Misspecification.} If $\bar q$ is not representable by the head, the population objective is, in the range parameters, a concave quadratic; its maximizer over the box $[0,1]^2$ is the constrained $L^2(p)$ projection of $\bar q$, unique and attracting under the projected flow whenever $\phi$ is nonconstant on the support of $p$ (then $\E[\nabla c\,\nabla c^\top]\succ0$), and interior exactly when no box constraint is active; on the deployment disagreement set $\{\sign(c-t^\ast)\neq\sign(\bar q-t^\ast)\}$ the threshold lies between $c$ and $\bar q$, so $|m(\bar q)|=(1+\lambda)|\bar q-t^\ast|\le(1+\lambda)|\bar q-c|$ pointwise, and the deployed-utility gap to Chow's rule is at most $(1+\lambda)\,\E_p|c-\bar q|$. The structural point: at the action level, reward and anchor share the vanishing factor $\sgm'(v)$; at the report level they share none: the Brier gradient $2(\bar q-c)$ is $\Theta(1)$ away from calibration, and the report-KL derivative \emph{diverges} at the boundary, repelling degenerate reports; the direction of the KL, load-bearing at the action level, is moot here. Scope: these are statements about the report \emph{value}; for a sampled confidence token the logit-space caveat of Section~\ref{app:static} applies, and overconfident initialization is the condition that discharges it. \qed

\paragraph{Frozen-feature training-properness, with an explicit clock.} Freeze $(w,a)$ at any values with non-constant $\phi$ and flow $x=(\kappa_0,\kappa_1)$ under $B_\beta(x)=B(x)-\tfrac\beta2\|x-x_0\|^2$. $B$ is a concave quadratic with Hessian $-2M$, $M=\E_p[uu^\top]\succ0$ for $u=(1-\phi,\phi)^\top$, so $B_\beta$ is $(2\lambda_{\min}(M)+\beta)$-strongly concave and the flow converges to its unique maximizer from every initialization in $\R^2$ at rate $e^{-(2\lambda_{\min}(M)+\beta)t}$, with $\|x^\ast_\beta-x^\ast\|\le\beta\|x^\ast-x_0\|/(2\lambda_{\min}(M))$ (a projected variant on $[0,1]^2$ keeps $c\in[0,1]$ throughout at the same rate: projection onto a convex set is nonexpansive, so the strong-concavity contraction is preserved). On the working model $2\lambda_{\min}(M)=0.55$, so the matched overconfident initialization reaches $10^{-3}$ parameter error by $T\approx11$, against the action-level rule, which the \emph{same} frozen-feature reduction leaves collapsing with recovery only at $e^{\Theta(1/\beta)}$. The mechanism separation is thus \emph{proved}, not simulated, in the frozen-feature subfamily; the joint four-parameter basin is nonconvex and remains an empirical statement (Section~\ref{app:sim}). \qed

\paragraph{Proof of Proposition 9 (composite objective).} $\E[(Y-c)^2]=q-2qc+c^2$ gives $\E[r_\alpha]=\alpha q+1-q+2qc-c^2$, whence the displayed partials. (a) $\partial_c=2(q-c)$ drives the report to calibration for every $\alpha$. (b) At $\alpha=0$ the calibrated score is $1-q+q^2$, U-shaped with minimum at $q=\tfrac12$: the accuracy gradient $2q-1$ is negative for every $q<\tfrac12$; and $q<\tfrac12$ is, up to the value of $t^\ast$, the region (B1) selects: the defect of a bare proper score lands on the same questions as the penalty rule's pathology. (c) On the calibration manifold the accuracy gradient is $\alpha-1+2q$; at the accuracy vertex $q=c=0$ it equals $\alpha-1$: attracting iff $\alpha<1$, degenerate at $\alpha=1$, repelling iff $\alpha>1$. This is the same local vertex test as Proposition~3, applied in the $q$ coordinate to our own proposal, which is why we require $\alpha>1$ strictly. The composite has no discrete abstain action, every episode returns a $q$-correlated score, and the added $\alpha q$ term is independent of $c$, leaving the report-vertex analysis unchanged. \qed

\paragraph{The clipping ablation, quantified.} Clipping the Brier training signal below $t^\ast$ (constant score, zero gradient there) reinstates the dead zone: simulation exhibits a sliding-mode equilibrium on the threshold (learned confidence terminating at $0.7499$ against $t^\ast=0.750$ at moderate $\beta$, and above it at large $\beta$, where the model wrongly answers), and deployment utility pointwise weakly worse than unclipped at every $\beta$ tested, though the rule remains proper throughout, if no longer strictly: below the clip all reports tie, so truth-telling is still a maximizer, just not the unique one. The operative variable is the reachable region of identically zero training gradient, not properness.

\section{Negative Results and Degenerate Cases}\label{app:negative}

\paragraph{Proof of Proposition 4 (proximal floor).} With the proximal anchor, $\dot c_0=e^{c_0}(-D_{\mathrm{eff}})(1+o(1))+\beta(c_0(0)-c_0)$; the restoring term does not vanish as $c_0\to-\infty$, and setting $\dot c_0=0$ gives $e^{c_0}=\beta(c_0(0)-c_0)/D_{\mathrm{eff}}$, a stable root at $c_0=\log\beta+\log\log(1/\beta)+O(1)$, approached exponentially fast; then $P_\infty=\Theta(\beta\log(1/\beta))$ with bounded tilt. \emph{The surrogate is not conservative: true KL is more permissive of collapse than the weight-space stand-in.} \qed

\paragraph{A norm--sharpness impossibility statement, and why it is demoted.} Under the proximal surrogate, stationarity bounds $\|\theta_e-\theta_0\|\le G_1/\beta$, and an $\varepsilon$-optimal stochastic policy needs logit norm $\Theta(\log(1/\varepsilon))$, so no equilibrium beats $\Delta(\beta)=\frac{\gamma\delta_0^2}{8}e^{-\delta_0G_1/\beta}$-suboptimality. Both halves fail as a statement about RLHF: (i) under \emph{true} KL the norm budget does not exist (Bernoulli KL is bounded uniformly in $\|\theta\|$, so hard-threshold policies are legitimate limiting stationary points); (ii) even under the surrogate the bound is numerically vacuous at realistic anchoring ($\Delta=2.3\times10^{-3}$ against an actual loss $0.527$ at $\beta=0.3$; $\Delta\approx10^{-22}$ at $\beta=10^{-2}$). All of the interesting failure is dynamical. We retain the statement because the way it fails under true KL is what pointed us to Lemma~1.

\paragraph{Parameterization-dependence of the exponent.} The $1/t$ law is a property of the logit parameterization class: any fixed linear or redundant softmax-logit reparameterization changes only constants (flowing both logits rather than their difference halves the envelope constants), while the nonlinear $v=-\log\theta$, $\theta>0$, gives $\dot\theta\propto(1+\theta)^{-2}$ and $\mathcal G\asymp t^{-1/3}$. All exponent claims are therefore stated at the standard parameterization.

\section{Simulation Details}\label{app:sim}

\paragraph{Calibration.} Two latent types $(q_H,q_L)=(0.9,0.2)$ ($\E[q]=0.55$), signal $s\mid H\sim N(\mu,1)$, $s\mid L\sim N(-\mu,1)$; $\mu$ is the base model's self-knowledge quality (the main text's weak, medium, and strong signal levels are $\mu=0.5,1,2$). The calibration is deliberately favorable to the penalty rule and held fixed; the collapse threshold is crossed by varying $\lambda$, and (B1) $\iff\E[q]<t^\ast$ fails at $\lambda=1$ for this calibration (so the $\lambda=1$ rows of the deployment table show the no-collapse regime, a property of the chosen accuracy rather than of $\lambda=1$; the empirical anchoring across deployed models is Section~\ref{app:empirical}).

\paragraph{Matched comparison.} Every panel uses: true gate-level policy KL for the penalty arm; the proximal anchor the report arm structurally admits (it has no action distribution, hence no policy KL; the asymmetry is structural, not a choice, and function-space anchors leave the conclusions unchanged, Section~\ref{app:report}); identical four-parameter bounded heads; matched overconfident, weakly discriminative initialization (base answers $\approx88\%$; reported confidences in $[0.70,0.95]$ against true posteriors in $[0.2,0.9]$); greedy deployment for both; horizon $T=10^4$; $\beta=10^{-3}$ unless stated. Reported quantities are absolute utility and coverage, never the ratio $U/U^\ast$ (meaningless where $U<0$ or $U^\ast\approx0$). The horizon-free phase diagram over $(\mu,\lambda,\beta)$ is Figure~\ref{sfig:fair}.

\begin{figure}[t]\centering
\includegraphics[width=0.97\linewidth]{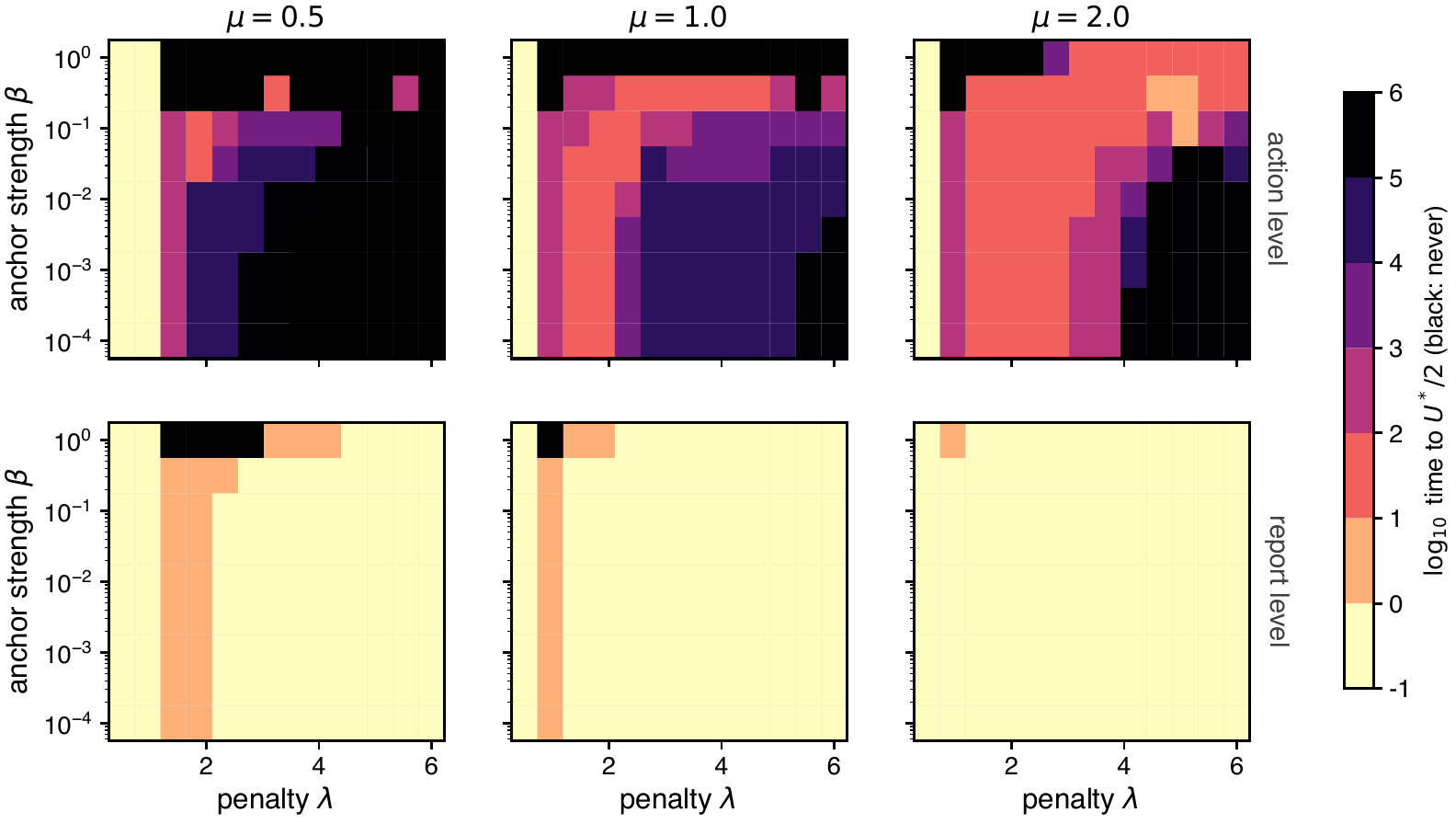}
\caption{Horizon-free phase diagram over $(\mu,\lambda)$. Colour is $\log_{10}$ of the time to competence, the first time the greedily deployed policy attains half of $U^\ast$; black cells never become competent within $10^6$. Top: action-level abstention under true policy KL. Bottom: report-level abstention. Both greedy, identical bounded heads; the two rows' $\beta$ multiply different objects (policy KL above, weight-space proximal term below) because a deterministic report policy has no action distribution; the asymmetry is structural, not a choice.}
\label{sfig:fair}
\end{figure}

\paragraph{Headline outcomes.} (i) The invariant $L$ for $\beta\in\{10^{-4},10^{-3},10^{-2}\}$ falls onto one curve, below the $\beta=0$ envelope past the transient ($Lt=1.26$ at $t=30$ against the intercept-carrying bound $1.53$; $1.14,1.05,1.01,1.002$ bounds against measured $0.71,0.57,0.51,0.32$ at $t=10^2,3\times10^2,10^3,6\times10^3$; the anchored guarantee carries the $\tfrac12$ and tends to $2$, the measured anchor share never exceeding $0.008$). (ii) Estimator contrast at $G=256$ on a homogeneous calibration where all estimators collapse: $\mathrm d\log L/\mathrm d\log t=-1.490$ (vanilla) and $-2.503$ (GRPO); both absolute values carry the same finite-time depression, and their \emph{difference} is $-1.013$ against a predicted $-1$; the difference, not the absolute exponent, is the stable statistic. (iii) Deployment at $\lambda=4.5$, $T=10^4$: the action-level arm deploys utility $0.000$ and coverage $0.000$ at every signal quality $\mu\in\{0.5,1,2\}$, forfeiting up to $U^\ast=0.199$ and $46.9\%$ coverage; the report-level arm attains $0.0045/0.074/0.194$ against optima $0.0045/0.0738/0.1994$ ($94$--$100\%$), with coverage within $2$ points of rational. (iv) At $\lambda=1$ ((B1) fails) both arms match the optimum: collapse is a property of the (B1) regime, not of training per se. (v) The unbounded-head control escapes by $t\approx10^4$ and recovers $U^\ast$ exactly, while the bounded head shows zero greedy coverage out to $10^7$ at $\mu=0.5$: the bounded readout is what makes the plateau consequential. (vi)~The anchor is load-bearing on the action side: swapping the true KL for the weight-space proximal term rescues deployment at $\mu\in\{1,2\}$ (coverage $0.234/0.468$ against rational $0.259/0.469$, $U=0.073/0.199$) but not at $\mu=0.5$ (coverage $0.000$ against rational $0.035$), measured at $T=10^4$, $\beta=10^{-3}$.

\section{An Empirical Sufficient Condition on Public Leaderboards}\label{app:empirical}

The blanket-answering condition (B1) is $\E[q]<t^\ast$, which public leaderboards do not identify: they report accuracy \emph{conditional} on the model's own abstentions, bounding $\E[q]$ only within $[\mathrm{acc},\mathrm{acc}+\mathrm{abst}]$. What is identified is a sufficient condition for the \emph{initial drift}. The drift at initialization is $\dot c_0(0)=\int\sgm'(v_0(s))g(s)\,ds$, and $\sgm'$ concentrates on the base policy's \emph{marginal} questions, not its confident ones. Suppose the base policy is describable by a sharp monotone confidence threshold, $v_0(s)=k(\bar q(s)-\theta_0)$ with $k$ large enough that $\sgm'(v_0)$ localizes on $\{\bar q\approx\theta_0\}$ at a scale on which $g$ varies little (a Laplace-type concentration condition; with monotone $\bar q$ the sign conclusion holds in the $k\to\infty$ limit, and we \emph{assume} it at finite $k$; this is the least rigorous step in this section). Then $\sign\dot c_0(0)=\sign(\theta_0-t^\ast)$, and since $\bar q>\theta_0$ on the attempted set,
\[
\theta_0\ \le\ \E[\bar q\mid\text{attempted}]\ =\ \frac{\mathrm{acc}}{\mathrm{acc}+\mathrm{err}} .
\]
Combining with $\mathrm{acc}-\lambda\,\mathrm{err}>0\iff\frac{\mathrm{acc}}{\mathrm{acc}+\mathrm{err}}>t^\ast$: \textbf{a negative error-penalized score (conditional accuracy below the Chow threshold) implies, within the threshold class, that the initial drift of penalty-rule training points toward collapse.} (Figure~\ref{sfig:money}.) The implication is one-directional, and the converse genuinely fails (a base with $\theta_0=0.30$, $k=30$ has conditional accuracy $0.532>t^\ast$ and a positive score, yet negative drift): the count \emph{undercounts} the affected set. On the AA-Omniscience leaderboard \citep{aa2025omniscience}, which scores $+1/-1/0$ over more than $36$ frontier models, only three scored above zero at the launch snapshot: at least $33$ satisfy the observable half of the condition at $\lambda=1$. Because $\lambda_{\mathrm{crit}}=\mathrm{acc}/\mathrm{err}$ varies across models (e.g.\ $1.17$, $1.00$, $0.79$, $0.73$ for the four models whose triples we analyzed), no single $\lambda$ is training-safe for a heterogeneous population. Scope: the behavioral half, that refusals follow a sharp monotone confidence threshold, is an assumption; models whose abstentions are driven by formatting or safety policies sit outside the certificate, and the benchmark itself is a measurement instrument whose rankings this analysis does not question.

\begin{figure}[t]\centering
\includegraphics[width=0.97\linewidth]{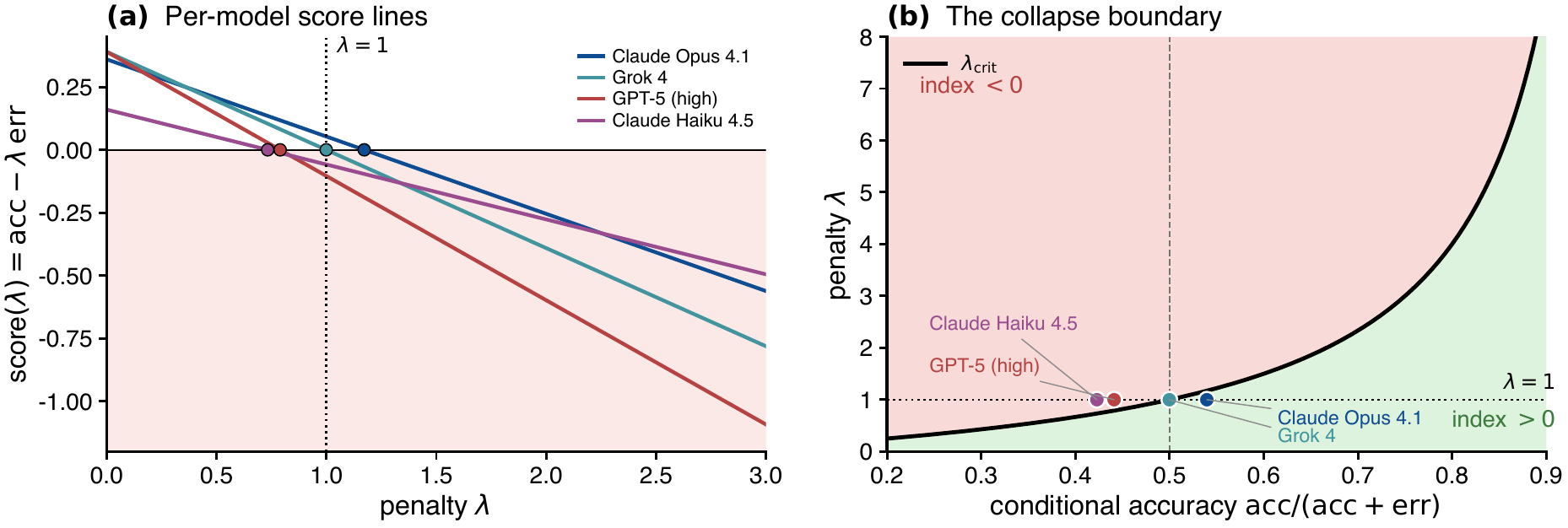}
\caption{(a)~Each model's score line $\mathrm{acc}-\lambda\,\mathrm{err}$; its zero is $\lambda_{\mathrm{crit}}=\mathrm{acc}/\mathrm{err}$, and at the deployed $\lambda=1$ the sign of the score is the published index. (b)~The same fact in the (conditional accuracy, $\lambda$) plane: a negative index \emph{implies} collapse drift; a positive one does not preclude it.}
\label{sfig:money}
\end{figure}

\paragraph{Re-scoring published evaluations.} Re-scoring published (accuracy, error, abstention) triples under $\mathrm{score}(\lambda)=\mathrm{acc}-\lambda\,\mathrm{err}$ across four public cohorts yields eight pairwise rank reversals at $\lambda<2.3$; the informative ones cross at $\lambda=0.041$ and $\lambda=0.834$ (a $16\%$-accuracy model overtaking a $39\%$-accuracy one). Reversals as such are guaranteed (score lines with different slopes must cross somewhere); what the re-scoring measures is \emph{where} they fall, and both informative crossings sit below the deployed $\lambda=1$: a leaderboard that fixes a single $\lambda$ is implicitly choosing among rankings that flip within the range of stakes it already spans.

\section{Language-Model Experiment Details}\label{app:llm}

Base models: Qwen2.5-1.5B (all multi-seed runs) and Qwen2.5-7B (single-seed runs at scale, with its own calibration of the pools and heads) \citep{qwen25}; short-form QA pools drawn from TriviaQA and PopQA \citep{joshi2017,mallen2023}, graded programmatically against reference answers, with a grading-error audit (manual audit of auto-wrong and auto-correct samples per cell; error masses $\delta_{\mathrm{FN}},\delta_{\mathrm{FP}}$ reported with confidence intervals, and the margin correction $m_{\mathrm{true}}=m_{\mathrm{obs}}+(1+\lambda)(\delta_{\mathrm{FN}}-\delta_{\mathrm{FP}})$ applied as a sensitivity analysis; conclusions whose sign depends on the correction are labeled grader-sensitive).

\paragraph{Infrastructure.} All runs are single-node, single-GPU jobs on a Slurm-managed cluster, each on one NVIDIA A100-SXM4-80GB (primary) or H100 80GB GPU, with $96$\,GB host memory and $8$ CPU cores per GPU job, Rocky Linux 8.10, Python 3.9.9, PyTorch 2.8.0 (CUDA 12.8, cuDNN 9.10.2), transformers 4.57.6, tokenizers 0.22.2, datasets 4.5.0, accelerate 1.10.1, numpy 2.0.2, scipy 1.13.1. Backbones (Qwen2.5-1.5B/7B-Instruct) run in bfloat16 with dropout disabled and log-probabilities computed in fp32; sampling is neutralized against the shipped generation configurations so rollouts come from the policy itself. GPU jobs total $637.5$ GPU-hours (longest single $8000$-step run ${\approx}35$\,h wall-clock); Tier-1 runs never call the language model and reproduce on CPU in minutes. Checkpointing is exact-resume, with randomness keyed per (seed, step).

\subsection{Tier 1: Head-Only Gate Dynamics on Frozen Features}
The paper's four-parameter gate head $v_\theta(x)=c_4\sgm(wz(x)+a)+c_0$ is trained on a frozen scalar feature $z(x)$ of the base model (the position-one logit difference between the two single-token decision symbols under an answer/abstain instruction template), with the answer/abstain decision implemented as an explicit two-token gate: the first generated position is restricted to the two decision symbols, and the per-prompt answer probability is read \emph{exactly} from the masked softmax, with no sampling error and no zero-inflation. Content and rewards come from pre-generated, pre-graded answer banks, so that $L(t)$, $\bar\sgm(t)$, $\rho(t)$, and $D_{\mathrm{eff}}(t)$ are exact pool sums at every step, and the reference distribution for the gate KL is the base head under the identical mask. This tier tests: (i) the integral envelope of Theorem~1(i), both as a deterministic expected-gradient flow (numerically integrated ODE; theorem-grade, tolerance-only) and under SGD sampling (sampling-robustness, multi-seed); (ii) the estimator menu, by sweeping the target per-prompt answer rate across half-decade points $10^{-5}$--$10^{-1}$ via bias offsets, at $G\in\{64,256\}$ and $\lambda\in\{1,4,9\}$, for vanilla policy gradient, mean-baseline, and group-standard-deviation normalization, against \emph{zero-parameter enumerated references} computed from the pool's own empirical correctness distribution. The tolerances: a GRPO-minus-vanilla local slope difference outside $[-1.3,-0.7]$ above the knee falsifies the $-2$ prediction; a measured knee outside a factor $2$ of the enumerated location, or a $G$-scaling ratio outside $[2,8]$ for $G=64{\to}256$, falsifies the knee prediction; below the knee, drift$/p$ outside $[2/3,1.5]$ of $(2\bar q-1)\sqrt{G-1}$, or a $\lambda$-ratio ($\lambda=9$ vs $4$) outside $[0.8,1.25]$, falsifies the $\lambda_{\mathrm{eff}}=1$ erasure; the dynamic-sampling comparison at $\beta=0$ tests rollout-axis coincidence. Raw-gradient probes at frozen checkpoints (fixed prompt subsets, fixed rollout counts) provide the optimizer-independent drift statistic.

\subsection{Tier 2: Full-Parameter Fine-Tuning}
Full-parameter RL fine-tuning with an \emph{unbiased} estimator (leave-one-out baselines) on a two-tier prompt mixture, $80\%$ low-accuracy ($\hat q\in[0.15,0.35]$) and $20\%$ high-accuracy ($\hat q\ge0.90$), chosen so that the aggregate satisfies (B1) while a certified subpopulation is unambiguously worth answering. The sequence-level KL to the frozen reference is implemented as a reward-side term (the summed per-token log-ratio) through its own leave-one-out advantage channel, i.e.\ the update direction is $\frac1G\sum_i(A_i^{r}-\beta A_i^{K})\nabla\log\pi_\theta(y_i\mid x)$ with separate baselines for the reward and KL channels; this is an unbiased gradient of the anchored objective \citep{williams1992,schulman2020kl,ahmadian2024}; the reference is masked identically to the policy at the gate position, and the forced end-of-sequence token after a refusal is excluded from the KL term.

\emph{Action-level arm.} The penalty rule $(+1,-\lambda,0)$ with the two-token decision interface of Tier 1, now carried by the full model's own next-token distribution with no added parameters: the first generated position is restricted to the answer/refusal pair, refusal is a single discrete action (immediate end of sequence), and the per-prompt answer probability is read exactly from the renormalized pair in one forward pass. Predictions: mean training reward negative and rising (the collapse signature), per-prompt answer rates on the high-accuracy tier crushed despite intact capability.

\emph{Capability certification.} The central confound, that the gate closed because the model got worse, is excluded by construction: at the crossing checkpoints and at the end of training, each high-tier prompt receives $256$ forced-answer rollouts, and the analysis conditions on the subset whose per-prompt one-sided $95\%$ Clopper--Pearson lower confidence bound on correctness exceeds $t^\ast$ at every capability checkpoint (retaining a positive answering margin). Drag-down is claimed only if the median gate probability on that certified subset falls below half its initial value at consecutive checkpoints.

\emph{Attribution control.} A paired run (common random numbers) in which the low-accuracy tier's \emph{reward-channel} advantages are zeroed while its KL channel is retained: if the high-tier gate collapse disappears, the collapse is attributed to the shared readout transmitting the low tier's negative drift (the mechanism) rather than to any high-tier-local force.

\emph{Report-level arm.} Same data, same estimator, no abstain action (the first position is forced to the answer symbol): a prompt-level confidence head is trained with the composite objective ($\alpha=2$; score channel through the group advantage on the composite reward with the report value detached, plus the pathwise Brier gradient through the head), the sequence KL applied to answer content only; abstention is applied only at deployment by thresholding the learned confidence at $t^\ast$. Predictions: no decay of mean reward under either estimator; final accuracy within $0.05$ and Brier score within $0.05$ of their early-training values (translation-invariant, component-wise criteria); deployment coverage and utility near the certified-capability optimum.

\subsection{Results}
\paragraph{Head-only gate dynamics.} Every quantitative tolerance above is met:\smallskip

\noindent{\small\setlength{\tabcolsep}{5pt}\begin{tabular}{@{}lcc@{}}
\toprule
quantity & theory & measured\\
\midrule
slope above the knee (group-norm.) & $\tfrac12$ & $0.507$\\
slope, mean-baseline & $1$ & $0.951$\\
knee scaling, $G=64{\to}256$ & ${\approx}4$ & $3.79$\\
knee vs.\ enumerated location & $1$ & $0.989$\\
\bottomrule
\end{tabular}}
\smallskip

\noindent Below the knee the drift is $\lambda$-invariant at every grid point, the amplitude ratios above the knee fall within their tolerance bands, and the first-window mean training reward rises in every sweep configuration ($p\le7{\times}10^{-4}$). The integral envelope of Theorem~1(i) holds along the exact expected-gradient flow and under SGD sampling in every seed. Two comparisons are structurally unavailable on this pool: the trajectory-slope contrast (the sparse-regime collapse leaves no common fitting window) and the dynamic-sampling comparison (the real base gate is bimodal, so zero-variance groups exhaust the resampling budget from the first step). No measured quantity contradicts a prediction.

\paragraph{Full-model collapse, per seed.} The certified cohort's median gate probability falls below half its initial value at consecutive checkpoints in three of three seeds, with $461/433/432$ of the $500$ prompts certified; the paired control's median never crosses and stays above $0.8$ of its start in three of three pairs. Step-$10$ cohort medians are $0.0059/0.0076/0.0037$ and terminal medians $5{\times}10^{-7}/1.6{\times}10^{-5}/3.7{\times}10^{-7}$, while the control's cohort median is exactly $1.000$ at every probe point. Cohort mean forced-answer correctness at the four capability checkpoints: $0.959\to0.965\to0.967\to0.939$. Capability at the two crossing checkpoints is measured at the exact steps, from replayed checkpoints whose weights are bit-identical to the run's: the per-prompt bound clears $t^\ast$ for $498/486/495$ of the $500$ prompts at step $10$ and $497/491/495$ at step $13$, with cohort mean correctness $0.95$--$0.97$ there.
Terminal diagnostic-batch gradient norms per seed: task $0/0/0$; gate-KL anchor $6{\times}10^{-13}/1.8{\times}10^{-10}/2.6{\times}10^{-13}$; the control's task-gradient norm at the same step is $108$. The mean reward rises over the first window in two of three seeds; the third collapses before the window closes. At 7B (single seed), the base answers $0.51$ of its answerable questions when offered the abstain option ($0.91$ without it); the cohort median collapses by step $10$, sits near $10^{-4}$ for roughly $6000$ steps, and leaves the plateau between the $6310$ and $7943$ probe points, recovering to terminal median $1.0$ with mean $0.65$ and a revived task gradient of $131$; the paired control rises from the same $0.555$ start to $0.930$ by step $10$ and $1.000$ thereafter, never dipping.

\paragraph{Report-level runs.} At 1.5B (three seeds each), a $3000$-step AdamW run (lr $10^{-6}$; an optimizer-robustness variant of the same objective) moves coverage $0.466\to0.598/0.604/0.602$ and accuracy $0.41\to0.57$ with Brier flat near $0.27$, and the $8000$-step plain-SGD run reaches coverage $0.778/0.810/0.794$, accuracy $0.685$, Brier $0.22$. At 7B (single seed) coverage passes $0.8$ by step $631$ and ends at $0.976$, accuracy $0.942$, Brier $0.052$ at step $8000$.
The 7B confidence head fits its calibration split $9{\times}$ better than at 1.5B (MSE $0.0025$ vs.\ $0.0220$) yet still opens at coverage $0.422$: linear-probe ranking improves with scale, absolute calibration at $t^\ast=0.818$ does not, so the readout remains the binding constraint at both scales.

\paragraph{Grading.} Answers are graded programmatically. The manual audit re-judged all $24$ cells (penalty and attribution-control arms $\times$ three seeds $\times$ two checkpoints $\times$ two tiers), $100$ auto-wrong and $50$ auto-correct answers per cell, $3600$ rows in total. The error profile is one-sided: false negatives (aliases, prefixes, sentence-wrapped and cross-language answers the string matcher missed) carry per-cell mass $\delta_{\mathrm{FN}}\in[0.002,0.033]$ on the high tier and $[0.046,0.131]$ on the low tier, while false positives are near-absent ($\delta_{\mathrm{FP}}\le0.020$ per cell, zero in $18$ of $24$). Measured correctness therefore understates true correctness throughout: the pooled margin correction $(1+\lambda)(\delta_{\mathrm{FN}}-\delta_{\mathrm{FP}})$ is $+0.058$ to $+0.161$ on three high-tier strata ($-0.057$ on the fourth) and $+0.376$ to $+0.619$ on the low tier; the capability certificates are conservative ($\delta_{\mathrm{FP}}\approx0$ leaves them no inflation channel); and the corrected low-tier accuracy (${\approx}0.30$--$0.35$ against a measured $0.23$) leaves (B1) a wide margin, the corrected mixture mean ${\approx}0.46$ sitting far below $t^\ast=0.818$. With $95\%$ Clopper--Pearson intervals on the audited fractions and unsure judgments counted both ways, no margin's sign flips within any cell's band: no conclusion is grader-sensitive.

\end{document}